\documentclass{article}

\usepackage{arxiv}

\usepackage[utf8]{inputenc}
\usepackage[T1]{fontenc}
\usepackage{url}
\usepackage{booktabs}
\usepackage{array}
\usepackage{amsmath, amssymb, amsfonts, amsthm}
\usepackage{nicefrac}
\usepackage{microtype}
\usepackage{graphicx}
\usepackage{natbib}
\usepackage{xcolor}
\usepackage{placeins}
\usepackage{float}
\newfloat{algorithm}{tbp}{loa}
\floatname{algorithm}{Algorithm}
\usepackage{hyperref}
\usepackage{doi}

\title{RODE: A Radial-Orthogonal Decoupled Engine for Optimization}

\author{
	Guoxiang Xu \quad Bince Qu \quad Qi Sun\textsuperscript{*} \quad Cheng Zhuo\textsuperscript{*} \\
	\normalfont Zhejiang University, Hangzhou, China \\
	\small\texttt{\{12647046,12621175\}@zju.edu.cn} \quad
	\textsuperscript{*}Correspondence: \texttt{\{qisunchn,czhuo\}@zju.edu.cn}
}

\renewcommand{\shorttitle}{RODE: A Radial-Orthogonal Decoupled Engine for Optimization}

\hypersetup{
hidelinks,
breaklinks=true,
pdftitle={RODE: A Radial-Orthogonal Decoupled Engine for Optimization},
pdfsubject={Optimization for deep neural networks},
pdfauthor={Guoxiang Xu, Bince Qu, Qi Sun, Cheng Zhuo},
pdfkeywords={optimization, matrix optimization, Newton-Schulz, geodesic update, radial decomposition},
}
\newcommand{\R}{\mathbb{R}}
\newcommand{\E}{\mathbb{E}}
\newcommand{\fro}{\mathrm{F}}
\newcommand{\tr}{\mathrm{tr}}
\newcommand{\Polar}{\mathrm{Polar}}
\newcommand{\clip}{\mathrm{clip}}

\newcommand{\widefigwidth}{0.94\linewidth}
\newcommand{\panelwidth}{0.40\linewidth}
\newcommand{\panelgap}{\hspace{0.03\linewidth}}

\newtheorem{assumption}{Assumption}
\newtheorem{lemma}{Lemma}
\newtheorem{theorem}{Theorem}

\begin{document}
\maketitle

\begin{abstract}
Modern neural network training increasingly uses matrix-aware optimizers, yet their conditioned matrix step is typically added directly to the weight, jointly changing its norm and direction. This interaction matters because the current norm determines angular motion, while directional learning can drive norm growth and thereby alter later steps. We introduce RODE, which gives the radial and directional components separate update rules and step sizes. RODE explicitly updates the matrix Frobenius norm through a scalar radial rule, while its directional channel performs Newton--Schulz-conditioned updates in the tangent space. Controlled GPT-2 interventions show gains from both direct norm control and RODE's directional update. Across two language-modeling and two image-classification tasks, RODE outperforms both Muon variants in every direct comparison and ends with lower full-model norms. At 1.5B scale, using the learning rate transferred directly from the Qwen2-style LM sweep, RODE lowers loss from $4.145$ to $3.346$ and final global norm from $11964$ to $2183$ relative to Muon RMS, with fixed-radius RODE improving further. For Qwen3.5-9B full-parameter fine-tuning, all six optimizers use the same tuning budget and the same formal-training and evaluation settings; RODE outperforms both Muon variants on all four evaluation tasks and attains the highest mean on GSM8K and MATH-500. Thus, decoupling radial and directional dynamics offers a more effective and controllable approach to matrix optimization.
\end{abstract}


\section{Introduction}
\label{sec:introduction}

Modern neural networks are trained largely by updating weight matrices. Adam and AdamW adapt individual coordinates \citep{kingma2015adam,loshchilov2019decoupled}, whereas matrix-aware optimizers exploit two-dimensional structure; Muon, for example, applies Newton--Schulz iterations to momentum \citep{jordan2024muon,higham2008functions}. Despite different constructions, most ultimately add a matrix-valued step to the current weight,
\begin{equation}
    W_{t+1}=W_t+\Delta W_t.
\label{eq:additive_update}
\end{equation}
We refer to this operation as an \emph{additive matrix update}.

An additive matrix update has two geometrically distinct effects. Write
\begin{equation}
    W_t = \rho_t U_t, \qquad \rho_t = \|W_t\|_\fro, \qquad U_t = \frac{W_t}{\rho_t}, \qquad \|U_t\|_\fro = 1,
\label{eq:polar_split}
\end{equation}
where $U_t$ is direction and $\rho_t$ is scale. In normalized or scale-invariant blocks, rescaling may preserve the represented function but still change the angular motion of the next update \citep{vanlaarhoven2017l2,hoffer2018norm,heo2021adamp,kosson2024rotational}; elsewhere, scale can also affect the function directly. Thus norm and direction play different roles.

Let $Q_t$ denote a matrix update produced by a conditioner such as Muon, so that $W_{t+1}=W_t-\eta Q_t$. Then
\[
\|W_t-\eta Q_t\|_\fro^2=\rho_t^2-2\eta\rho_t\langle U_t,Q_t\rangle_\fro+\eta^2\|Q_t\|_\fro^2,
\]
The component parallel to $U_t$ changes Frobenius norm to first order, while $Q_t^\perp$ changes direction by $\eta\|Q_t^\perp\|_\fro/\rho_t+\mathcal O(\eta^2)$. Even if $Q_t$ is tangent, the quadratic term grows the squared norm. Directional learning can therefore increase norm and reduce the angle realized by later absolute steps. Muon's update is generally not tangent to $U_t$, so one conditioned step mixes both effects. Weight decay can oppose growth, but only through an optimizer- and schedule-dependent balance \citep{kosson2024rotational,wang2025adamw}. Explicit radial control instead makes this choice visible and tunable.

We introduce \textbf{RODE}, a \textbf{R}adial-\textbf{O}rthogonal \textbf{D}ecoupled \textbf{E}ngine for matrix optimization. RODE extracts radial and directional signals, gives them different update rules and learning rates, and retains Newton--Schulz conditioning for direction. It uses tangent projections and a spherical directional realization, while a scalar rule updates radius. Radius evolution and directional learning thus become separately adjustable, while momentum across changing tangent spaces introduces a residual interaction that is accounted for in our analysis (Appendix~\ref{app:global_convergence}).

RODE retains matrix-aware conditioning while making norm and directional dynamics explicitly controllable. We evaluate its components through controlled interventions, its overall optimization performance across four language and vision tasks, hyperparameter transfer at 1.5B scale, and full-parameter fine-tuning at 9B scale.

Our contributions are:
\begin{itemize}
    \item We introduce RODE, a matrix optimizer that reformulates additive matrix optimization in radial--directional coordinates, turning weight scale and direction from implicitly coupled consequences of an update into separately controllable optimization variables.

    \item First, RODE introduces an explicit radial optimization channel. It extracts the radial gradient of each managed matrix and updates its Frobenius norm with a dedicated scalar rule and learning rate, preventing norm evolution from being an uncontrolled by-product of directional learning.

    \item Second, RODE develops an orthogonal directional engine for matrix-aware optimization. It reprojects momentum onto the current tangent space, applies projected Newton--Schulz conditioning, and realizes the resulting update as an on-sphere directional step with an independent learning rate. This construction preserves matrix-aware conditioning while decoupling directional progress from the current parameter norm. Our analysis further accounts for the residual interaction induced by momentum across changing tangent spaces and provides convergence guarantees under explicit assumptions (Appendices~\ref{app:global_convergence} and~\ref{app:pl_convergence}).

    \item We validate both the mechanism and the resulting optimizer empirically. Controlled interventions separately identify the gains from explicit norm control and the directional engine, while experiments across language modeling and image classification, together with the 1.5B hyperparameter-transfer study, demonstrate consistent improvements over Muon variants with substantially more controlled norm growth. In Qwen3.5-9B full-parameter fine-tuning, all six optimizers use the same tuning budget and formal-training and evaluation settings; RODE outperforms both Muon variants on all four evaluation tasks and attains the highest mean on GSM8K and MATH-500.
\end{itemize}

\section{Related Work}
\label{sec:related}

RODE connects two often separate lines of work: matrix-valued update directions and parameter-scale control.

\paragraph{Adaptive and matrix-aware optimizers.}
Adam(W) combines momentum with coordinate-wise second moments \citep{kingma2015adam,loshchilov2019decoupled}; Adafactor factorizes that state \citep{shazeer2018adafactor}, and AdEMAMix mixes fast and slow momenta \citep{pagliardini2024ademamix}. Shampoo uses Kronecker-factored statistics \citep{gupta2018shampoo}; SOAP updates Adam-like statistics in the preconditioner's eigenbasis \citep{vyas2025soap}; and MARS-M adapts variance reduction to matrix updates \citep{yuan2024mars,liu2025marsm}. Muon applies short Newton--Schulz iterations to momentum \citep{jordan2024muon}, with later work studying its large-scale recipe \citep{liu2025muonscalable}. Muon adds its conditioned step directly; RODE gives radial and matrix-conditioned directional signals different rules.

\paragraph{Scale, direction, and manifold methods.}
Weight normalization exposes magnitude and direction in the model parameterization \citep{salimans2016weight}. For scale-invariant normalized weights, the norm controls the effective learning rate: update-induced norm growth slows subsequent directional change, and weight decay can counteract this effect \citep{vanlaarhoven2017l2,hoffer2018norm,roburin2022spherical}. AdamP removes norm-increasing radial momentum components to prevent premature effective-step decay \citep{heo2021adamp}. Rotational Equilibrium shows that update-induced growth and weight decay can settle into an equilibrium norm and angular update, and that controlling rotation explicitly recovers important benefits of that balance \citep{kosson2024rotational}. The appropriate AdamW decay also changes with model and dataset scale and can determine whether learning-rate transfer remains valid \citep{wang2025adamw}. Together, these results establish weight norm as a controller of optimization dynamics, not merely a regularization statistic.

LARS/LAMB scale additive steps by parameter norms \citep{you2017large,you2020large}; Fromage uses layer-relative updates \citep{bernstein2020fromage}; and Nero performs angle-controlled projected updates per neuron \citep{liu2021nero}. RODE instead exposes one matrix-wide radius as an optimizer coordinate and applies its orthogonalizing kernel to tangent-projected momentum. Riemannian optimization methods based on tangent-space updates and manifold-preserving maps motivate RODE's constrained directional update \citep{absil2008optimization,bonnabel2013stochastic,becigneul2019riemannian}.

\paragraph{Recent radius--direction methods.}
Hyperball wraps a base optimizer by constraining each managed matrix to the sphere defined by its initial Frobenius norm and normalizing the base-optimizer update, so the learning rate directly sets the relative update length \citep{wen2026hyperball}. The radius remains fixed throughout training. Magnitude--Direction Decoupling also maintains a fixed-norm directional factor, while scale is represented by separately optimized row and column gains; the directional factor can be updated by Adam, Muon, or another base optimizer \citep{hagele2026md}. In its default form, RODE learns one matrix-wide scalar radius directly from the radial gradient. Its directional channel applies projected Newton--Schulz conditioning to tangent-projected momentum. The radial and directional channels use independent learning rates.

\section{Method}
\label{sec:method}

RODE decomposes the current signal into radial and directional components, gives them different rules and step sizes, and retains matrix conditioning for direction. Figure~\ref{fig:rode_mechanism} contrasts this interface with a direct additive update. We describe one matrix; other parameters use AdamW.

\begin{figure}[t]
	\centering
	\includegraphics[width=\widefigwidth]{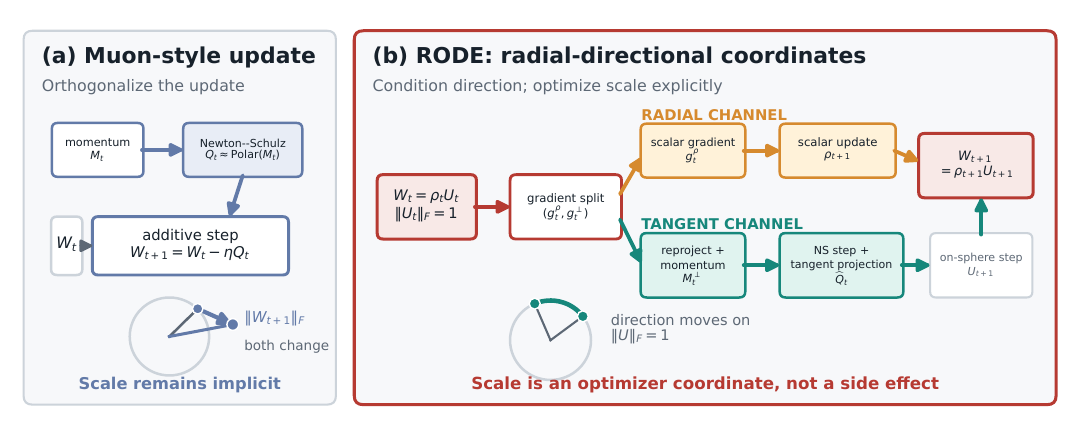}
		\caption{Muon conditions momentum and adds the resulting matrix step directly to $W_t$, so its effects on norm and direction are jointly determined. RODE exposes a scalar radial signal and a tangent-projected directional signal, assigns them separate step sizes, and reconstructs the managed matrix from the resulting $\rho_{t+1}$ and $U_{t+1}$.}
	\label{fig:rode_mechanism}
\end{figure}

\subsection{Radial--tangent gradient decomposition}

Let $g_t$ be a stochastic estimator of $G_t=\nabla f(W_t)$, and let $\langle A,B\rangle_\fro=\tr(A^\top B)$. We use $\mathcal P_t$ to denote the orthogonal projection onto the tangent space at the current direction $U_t$. Specifically,
\begin{align}
    \mathcal P_t(X)&=X-\langle X,U_t\rangle_\fro U_t, &
    g_t^\rho&=\langle g_t,U_t\rangle_\fro, &
    g_t^\perp&=\mathcal P_t(g_t).
\label{eq:gradient_split}
\end{align}
Then $\langle g_t^\perp,U_t\rangle_\fro=0$, so $g_t^\rho$ and $g_t^\perp$ are the instantaneous radial and tangent components of the current gradient. Under the parameterization $W_t=\rho_t U_t$, the gradient with respect to the directional coordinate $U$ is $\rho_t g_t^\perp$. RODE omits this scale factor when forming the directional signal, so that the directional update is not directly scaled by the current matrix norm. Since a positive scalar does not change the normalized polar direction used by the Newton--Schulz conditioner, $\eta_{\mathrm{dir}}$ can directly control the directional step.

\subsection{Radial update}

Let $B_{t-1}$ denote the raw momentum buffer stored from the previous step, and recall that $\mathcal P_t$ projects a matrix onto the tangent space at the current direction $U_t$. The radial channel updates only the matrix Frobenius norm, using the current gradient component parallel to $U_t$. RODE first reprojects the historical momentum onto the current tangent space and then adds the current gradient:
\begin{align}
    \widetilde B_t&=\beta\mathcal P_t(B_{t-1})+g_t, &
    s_t&=\langle\widetilde B_t,U_t\rangle_\fro=g_t^\rho, \nonumber\\
    \rho_{t+1}&=\max\{\rho_t-\eta_\rho s_t,\kappa_{\mathrm{floor}}\rho_t\}. &&
\label{eq:radial_update}
\end{align}
Because $\mathcal P_t(B_{t-1})$ lies in the tangent space at $U_t$, it is orthogonal to $U_t$ and therefore contributes no radial component. Hence the radial signal satisfies $s_t=g_t^\rho$ and is determined only by the current gradient. The scalar learning rate $\eta_\rho$ directly controls the change in the matrix Frobenius norm. The multiplicative floor prevents a single update from collapsing the radius or crossing zero; our analysis considers the interior regime in which this safeguard is inactive.

\subsection{Tangent momentum and conditioning}

RODE forms the tangent momentum by projecting the updated raw buffer,
\begin{equation}
    M_t^\perp=\mathcal P_t(\widetilde B_t),
\label{eq:tangent_momentum}
\end{equation}
while storing $B_t=\widetilde B_t$ for the next step. Because the direction $U_t$ changes during training, radial content retained in the raw buffer can enter a future tangent space after reprojection. Appendix~\ref{app:global_convergence} bounds this cross-step residual.

For wide matrices, RODE first transposes the matrix and initializes
\begin{equation}
    X_0=\frac{M_t^\perp}{\|M_t^\perp\|_\fro}.
\label{eq:ns_init}
\end{equation}
It then applies the projected Newton--Schulz iterations
\begin{align}
    A_k&=X_k^\top X_k, \nonumber\\
    X_{k+1}&=\mathcal P_t\!\left[
    3.4445X_k+
    X_k\{-4.7750A_k+2.0315A_k^2\}
    \right].
\label{eq:ns}
\end{align}
Newton--Schulz iterations approximate matrix sign and polar factors \citep{higham2008functions}. Each polynomial step can introduce a component parallel to $U_t$, so RODE projects the result back onto the current tangent space after every iteration. After $K$ iterations, we denote the resulting conditioned tangent kernel by
\begin{equation}
    \widehat Q_t=X_K.
\label{eq:conditioned_tangent_kernel}
\end{equation}
Its Frobenius norm is used in the directional update below.

\subsection{Spherical directional realization}

Let
\begin{equation}
    q_t=\|\widehat Q_t\|_\fro.
\label{eq:conditioned_kernel_norm}
\end{equation}
For $q_t>0$, define the normalized tangent direction $D_t^\perp$ and rotation angle $\theta_t$ as
\begin{align}
    D_t^\perp&=\widehat Q_t/q_t, &
    \theta_t&=\eta_{\mathrm{dir}}q_t, \nonumber\\
    U_{t+1}&=U_t\cos\theta_t-D_t^\perp\sin\theta_t, &
    W_{t+1}&=\rho_{t+1}U_{t+1}.
\label{eq:sphere_update}
\end{align}
Because $D_t^\perp$ lies in the tangent space at $U_t$, we have $\langle D_t^\perp,U_t\rangle_\fro=0$, which gives $\|U_{t+1}\|_\fro=1$ in exact arithmetic. The implementation renormalizes $U_{t+1}$ for numerical stability and handles the $q_t=0$ case separately; Appendix~\ref{app:update} gives pseudocode.

\subsection{Optional scalar damping}

RODE-SNR optionally damps the directional step using a running estimate of the tangent-gradient scale. We maintain a scalar state $v_t$ that accumulates the recent magnitude of $g_t^\perp$:
\begin{equation}
    v_t=\beta v_{t-1}+\|g_t^\perp\|_\fro+\epsilon.
\label{eq:snr_scale}
\end{equation}
The damping factor compares the current tangent-momentum norm with this reference scale:
\begin{equation}
    \gamma_t=\clip\!\left(
        \frac{\|M_t^\perp\|_\fro}{v_t},
        \gamma_{\min},1
    \right).
\label{eq:snr_damping}
\end{equation}
RODE-SNR then uses
\begin{equation}
    \theta_t=\eta_{\mathrm{dir}}\gamma_t q_t,
\label{eq:snr_angle}
\end{equation}
so $\gamma_t$ changes only the directional step magnitude and leaves $D_t^\perp$ unchanged. Default RODE sets $\gamma_t=1$; we report RODE-SNR as an optional, separately tuned variant.

\section{Theory}
\label{sec:theory}

Appendix~\ref{app:global_convergence} analyzes the implemented raw-buffer recurrence in the nondegenerate interior regime, where the radial floor is inactive and the projected tangent kernel remains nonzero. The analysis assumes smoothness, bounded stochastic noise and iterates, local tangent-gradient regularity, bounded rotation angles, and positive alignment between the projected Newton--Schulz kernel and the tangent momentum.

For the full gradient $G_t=\nabla f(W_t)$, let
$G_t^\rho=\langle G_t,U_t\rangle_\fro$ and
$G_t^\perp=\mathcal P_t(G_t)$ denote its radial and tangent components. Under the stated assumptions, an exact expansion of the displacement $W_{t+1}-W_t$ separates radial descent, tangent descent, geodesic curvature, and cross-step leakage from the raw momentum buffer. A Lyapunov function combining $f(W_t)$ with the tangent-momentum tracking error yields
\begin{align}
    \limsup_{T\to\infty}
    \frac{1}{T}\sum_{t=1}^{T}
    \E\!\left[(G_t^\rho)^2\right]
    &\le \epsilon_\rho, &
    \limsup_{T\to\infty}
    \frac{1}{T}\sum_{t=1}^{T}
    \E\!\left[\|G_t^\perp\|_\fro^2\right]
    &\le \epsilon_\perp.
\label{eq:neighborhood_summary}
\end{align}
The explicit expressions for $\epsilon_\rho$ and $\epsilon_\perp$ are given in Theorem~\ref{app:thm:neighborhood}; they depend on the radial and directional step sizes, the stochastic-noise levels, and the constants in the stated assumptions. These bounds cover both gradient components even though the stored raw momentum need not remain perfectly tangent as the direction $U_t$ changes.

Appendix~\ref{app:pl_convergence} separately studies the first-order approximation
\begin{equation}
    W_{t+1}=W_t-\eta H_t,
\label{eq:aligned_first_order_model}
\end{equation}
where $H_t$ is the effective update direction and $\eta>0$ is its step size. The approximation assumes that $H_t$ has uniformly bounded Frobenius norm and remains positively aligned with the full gradient:
\begin{equation}
    \langle G_t,H_t\rangle_\fro
    \ge
    \alpha\|G_t\|_\fro^2,
    \qquad \alpha>0,
\label{eq:effective_alignment}
\end{equation}
where $\alpha$ quantifies the minimum degree of alignment. The objective is also assumed to satisfy the Polyak--Lojasiewicz condition \citep{karimi2016pl},
\begin{equation}
    \frac{1}{2}\|G_t\|_\fro^2
    \ge
    \mu\bigl(f(W_t)-f^*\bigr),
    \qquad \mu>0,
\label{eq:pl_condition_summary}
\end{equation}
where $\mu$ is the PL constant and $f^*=\inf_W f(W)$ is the optimal objective value. If
$0<2\eta\alpha\mu<1$, smoothness gives
\begin{equation}
    f(W_t)-f^*
    \le
    (1-2\eta\alpha\mu)^t
    \bigl(f(W_0)-f^*\bigr)
    +\mathcal O(\eta).
\label{eq:pl_rate_summary}
\end{equation}
Thus the aligned first-order model converges linearly to a neighborhood whose size is controlled by $\eta$. Its contraction factor $1-2\eta\alpha\mu$ depends jointly on the alignment strength $\alpha$ and the PL constant $\mu$.

\section{Experiments}
\label{sec:experiments}

We test optimization quality and norm behavior, separate radial and directional contributions, then examine scale, modality, components, and systems cost.

\subsection{Experimental questions and setup}
\label{sec:setup}

We study two WikiText-103 \citep{merity2017pointer} decoders: standard GPT-2 \citep{radford2019language} for 8000 steps and a randomly initialized Qwen2-style language model for 4000 steps. The latter has 114M parameters and uses RoPE, RMSNorm, grouped-query attention, and SwiGLU \citep{qwen2report}; it is a locally instantiated architecture rather than an official Qwen checkpoint. Vision experiments train ResNet-50 \citep{he2016deep} on CIFAR-100 \citep{krizhevsky2009learning} and ImageNet-1K \citep{deng2009imagenet}. For a separate scaling study, we scale the Qwen2-style architecture to a 1.5B-parameter decoder and train it for 20000 steps on FineWeb \citep{penedo2024fineweb}. We additionally conduct a full-parameter fine-tuning benchmark on Qwen3.5-9B \citep{qwen35}; all six optimizers use the same tuning budget and formal-training and evaluation settings.

All reported optimization-quality comparisons use three seeds. In language pretraining, RODE manages two-dimensional non-embedding, non-head matrices; in vision it manages flattened multi-dimensional weights. Remaining parameters use AdamW. A within-task formal comparison shares the model, data, horizon, seed set, and checkpoint rule; optimizer-specific rates are selected by Optuna or transferred as marked. The original vision suite reuses standard-GPT-2 settings, added baselines reuse Qwen2-style-LM settings, and the 1.5B test reuses Qwen2-style-LM settings without another sweep. Appendix~\ref{app:hyperparams} specifies routing, search spaces, schedules, and selected values. Evaluation metrics and all tabulated values use raw logs.

\begin{table}[htbp]
	\caption{Cross-task optimization performance on language modeling and image classification benchmarks. Results are reported as mean $\pm$ sample standard deviation over three seeds. Vision reports best accuracy on the official CIFAR-100 test set and the ImageNet-1K validation set. \textsuperscript{$\dagger$}Rates transferred from the Qwen2-style LM; \textsuperscript{$\ddagger$}rates transferred from GPT-2. Best available mean per column is bold.}
	\centering
	\normalsize
	\setlength{\tabcolsep}{3pt}
	\begin{tabular}{lcccc}
		\toprule
		Optimizer & GPT-2 loss $\downarrow$ & Qwen2-style LM loss $\downarrow$ & CIFAR acc. $\uparrow$ & ImageNet acc. $\uparrow$ \\
		\midrule
		AdamW & $3.1792 \pm 0.0173$ & $3.0503 \pm 0.0071$ & $68.66 \pm 0.29$\textsuperscript{$\ddagger$} & $62.49 \pm 0.46$\textsuperscript{$\ddagger$} \\
		Muon Original & $3.0295 \pm 0.0001$ & $2.8185 \pm 0.0065$ & $71.49 \pm 0.67$\textsuperscript{$\ddagger$} & $64.98 \pm 0.05$\textsuperscript{$\ddagger$} \\
		Muon RMS & $3.0265 \pm 0.0027$ & $3.0205 \pm 0.0103$ & $71.11 \pm 0.14$\textsuperscript{$\ddagger$} & $64.48 \pm 0.14$\textsuperscript{$\ddagger$} \\
		SGD & -- & -- & $72.77 \pm 0.24$ & $70.25 \pm 0.08$ \\
		RODE & $2.9774 \pm 0.0011$ & $2.7731 \pm 0.0013$ & $\mathbf{73.75 \pm 0.29}$\textsuperscript{$\ddagger$} & $69.74 \pm 0.47$\textsuperscript{$\ddagger$} \\
		RODE-SNR & $2.9724 \pm 0.0557$ & $\mathbf{2.7493 \pm 0.0011}$ & $73.64 \pm 0.56$\textsuperscript{$\ddagger$} & $\mathbf{70.49 \pm 0.09}$\textsuperscript{$\ddagger$} \\
		SOAP & $\mathbf{2.9522 \pm 0.0022}$\textsuperscript{$\dagger$} & $2.7684 \pm 0.0008$ & $73.03 \pm 0.03$\textsuperscript{$\dagger$} & -- \\
		MARS-M & $2.9856 \pm 0.0009$\textsuperscript{$\dagger$} & $2.8885 \pm 0.0040$ & $66.26 \pm 0.63$\textsuperscript{$\dagger$} & -- \\
		AdEMAMix & $3.2088 \pm 0.0029$\textsuperscript{$\dagger$} & $2.9807 \pm 0.0026$ & $65.57 \pm 0.35$\textsuperscript{$\dagger$} & -- \\
		\bottomrule
	\end{tabular}
	\label{tab:main_results}
\end{table}

\subsection{Optimization quality and norm control}

RODE shows consistent gains across all four main benchmarks (Table~\ref{tab:main_results}). Compared with the stronger Muon variant on each task, it reduces loss by $0.0491$ on GPT-2 and $0.0454$ on the Qwen2-style LM, while improving accuracy by $2.36$ points on CIFAR-100 and $4.48$ points on ImageNet-1K. Its performance also remains competitive against the broader optimizer set: SOAP gives the best GPT-2 result, RODE-SNR leads on the Qwen2-style LM and ImageNet-1K, and RODE leads on CIFAR-100. All comparisons are made within task, with tuning and transfer sources indicated in Table~\ref{tab:main_results}.

The norm trajectories reveal a complementary aspect of RODE's behavior. Among non-SGD methods, a RODE variant attains the lowest observed final full-model norm on every task; on GPT-2, CIFAR-100, and ImageNet-1K, the two lowest values are both obtained by RODE variants (Table~\ref{tab:final_norms}). Muon RMS finishes at $2.5$--$5.7\times$ the norm of RODE. This pattern aligns with the motivation for explicit norm control: for scale-invariant weights, parameter norm governs effective directional progress, and sustained update-induced norm growth can progressively reduce subsequent angular steps \citep{vanlaarhoven2017l2,heo2021adamp,kosson2024rotational}. Across these experiments, improved task performance under RODE is accompanied by substantially more controlled norm evolution. Full-model norm includes all trainable parameters; Appendix~\ref{app:metric_definitions} gives the exact definition. Figure~\ref{fig:gpt2_main} shows the GPT-2 trajectories, with the Qwen2-style LM results reported in Appendix Figure~\ref{app:fig:qwen114m}.

\begin{table}[htbp]
	\caption{Final observed full-model weight norm at the end of training (three-seed mean $\pm$ sample SD). Values include every trainable parameter, including tensors delegated to AdamW, and are comparable only within a task. A dash indicates that the task--optimizer combination was not run; values are descriptive rather than ranked.}
	\label{tab:final_norms}
	\centering
	\normalsize
	\setlength{\tabcolsep}{4pt}
	\begin{tabular}{lcccc}
		\toprule
		Optimizer & GPT-2 & Qwen2-style LM & CIFAR-100 & ImageNet-1K \\
		\midrule
		AdamW & $413.6\pm3.0$ & $711.1\pm8.7$ & $355.9\pm0.7$ & $2049.5\pm34.7$ \\
		Muon Original & $1187.9\pm0.8$ & $1314.7\pm1.9$ & $1302.6\pm0.6$ & $4889.3\pm0.5$ \\
		Muon RMS & $1074.4\pm0.8$ & $1722.4\pm2.9$ & $1142.0\pm1.2$ & $4547.4\pm6.9$ \\
		SGD & -- & -- & $29.8\pm0.7$ & $56.6\pm0.1$ \\
		RODE & $315.6\pm6.6$ & $683.5\pm2.6$ & $348.6\pm0.3$ & $801.8\pm27.6$ \\
		RODE-SNR & $269.1\pm0.3$ & $695.8\pm4.1$ & $272.8\pm0.1$ & $320.6\pm0.1$ \\
		SOAP & $1440.1\pm1.9$ & $1046.9\pm2.6$ & $1136.0\pm1.7$ & -- \\
		MARS-M & $2220.0\pm8.0$ & $1417.6\pm3.8$ & $23540.5\pm470.3$ & -- \\
		AdEMAMix & $454.0\pm1.5$ & $695.6\pm2.4$ & $8375.5\pm974.9$ & -- \\
		\bottomrule
	\end{tabular}
\end{table}

\begin{figure}[tbp]
	\centering
	\begin{minipage}{\panelwidth}\centering\includegraphics[width=\linewidth]{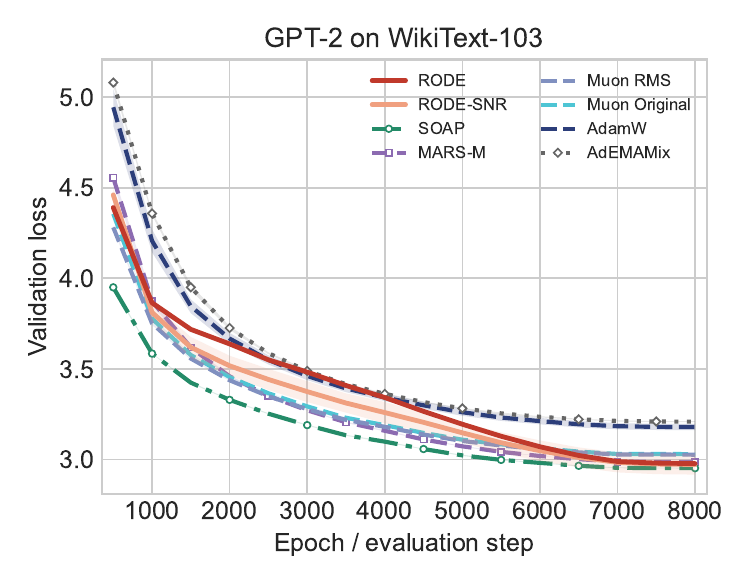}\end{minipage}\panelgap
	\begin{minipage}{\panelwidth}\centering\includegraphics[width=\linewidth]{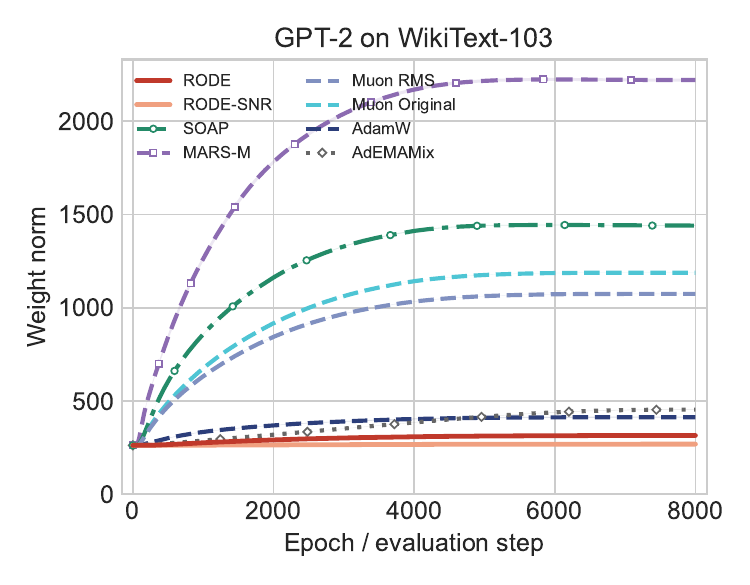}\end{minipage}
		\caption{Standard GPT-2/WikiText-103 over 8000 steps (three-seed mean $\pm$ sample SD): validation loss (left) and observed full-model weight norm (right). The norm aggregates all trainable parameters, including those delegated to AdamW; both RODE variants show substantially less growth than the Muon runs.}
	\label{fig:gpt2_main}
\end{figure}

\subsection{Separating norm control from directional conditioning}
\label{sec:norm_direction_intervention}

Lower norms alone do not establish an optimization benefit, so Table~\ref{tab:norm_direction_intervention} introduces the radial and directional interventions sequentially. Fixing the radii of Muon-managed matrices lowers validation loss from $3.027$ to $3.008$. Replacing Muon's additive directional update with RODE's tangent directional update further lowers the loss to $2.977$, while full RODE with learned radii reaches the same level. These interventions show a measurable benefit from explicit norm control and a further gain from RODE's directional update. Muon RMS and full RODE use learning rates selected from their respective GPT-2 sweeps; Muon + NormProj inherits the Muon RMS settings, and fixed-radius RODE inherits the full RODE settings.

\begin{table}[htbp]
	\caption{Sequential GPT-2 norm--direction intervention. Three-seed values are mean $\pm$ sample SD; $\Delta$ is relative to Muon RMS.}
	\label{tab:norm_direction_intervention}
	\centering
	\normalsize
	\setlength{\tabcolsep}{4pt}
	\begin{tabular}{llllc}
		\toprule
			Condition & Direction update & Radius policy & Val. loss $\downarrow$ & $\Delta$ vs.\ Muon \\
		\midrule
			Muon RMS & additive Muon & implicit & $3.027\pm0.003$ & -- \\
			Muon + NormProj & additive Muon & fixed by projection & $3.008\pm0.001$ & $-0.019$ \\
			RODE, fixed $\rho$ & RODE tangent stack & fixed & $2.977\pm0.002$ & $-0.050$ \\
			Full RODE & RODE tangent stack & learned scalar & $2.977\pm0.001$ & $-0.050$ \\
		\bottomrule
	\end{tabular}
\end{table}

\subsection{Hyperparameter transfer at 1.5B scale}

We train a Qwen2-style 1.5B decoder for 20000 FineWeb steps, transferring the smaller model's rates without a new sweep (Table~\ref{tab:scaling_main}). With the transferred rates, RODE achieves lower loss and a substantially lower global norm than Muon RMS. Fixed-radius RODE improves both measures further. Figure~\ref{fig:scaling_main} shows the trajectories; Appendices~\ref{app:scaling_results} and~\ref{app:hyperparams} give the cache construction and complete protocol.

\begin{table}[htbp]
	\caption{1.5B scale transfer with rates selected on the smaller Qwen2-style LM (three seeds, mean $\pm$ sample SD).}
	\label{tab:scaling_main}
	\centering
	\normalsize
	\begin{tabular}{lccc}
		\toprule
		Optimizer & Best cache loss $\downarrow$ & Final cache loss $\downarrow$ & Final observed global norm \\
		\midrule
		RODE, fixed $\rho$ & $\mathbf{3.13670\pm0.00406}$ & $\mathbf{3.13670\pm0.00406}$ & $1743\pm7$ \\
		RODE & $3.34568\pm0.01108$ & $3.34571\pm0.01109$ & $2183\pm2$ \\
		Muon RMS & $4.14465\pm0.01800$ & $4.14478\pm0.01806$ & $11964\pm32$ \\
		\bottomrule
	\end{tabular}
\end{table}

\begin{figure}[htbp]
	\centering
	\begin{minipage}{\panelwidth}\centering\includegraphics[width=\linewidth]{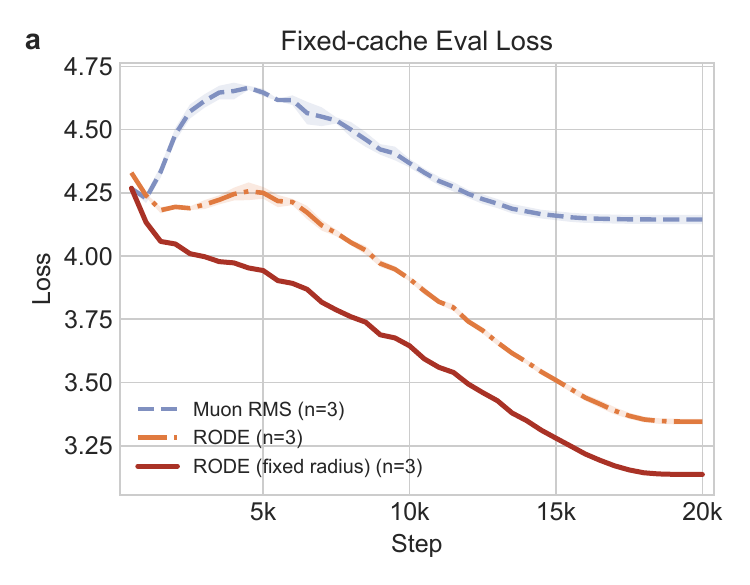}\end{minipage}\panelgap
	\begin{minipage}{\panelwidth}\centering\includegraphics[width=\linewidth]{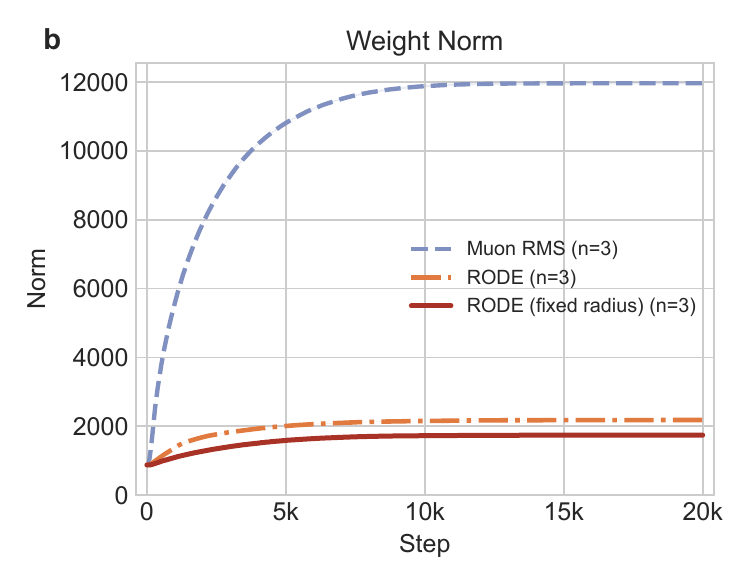}\end{minipage}
	\caption{Qwen2-style 1.5B/FineWeb scale transfer over 20000 steps (three-seed mean $\pm$ sample SD): unsmoothed fixed-cache loss and observed global weight norm. Fixed-radius RODE freezes managed matrix radii; parameters delegated to AdamW can still change the global norm.}
	\label{fig:scaling_main}
\end{figure}

\subsection{Full-parameter fine-tuning at 9B scale}
\label{sec:qwen35_9b_benchmark}

We full-parameter fine-tune Qwen3.5-9B on GSM8K \citep{cobbe2021gsm8k}. Each of the six optimizers receives the same tuning budget of 20 trials of 800 steps, followed by three 3000-step runs with distinct seeds; the minimum-validation-loss checkpoint from each run is evaluated on GSM8K, MATH-500 \citep{hendrycks2021math}, MMLU-STEM \citep{hendrycks2021mmlu}, and MMLU-Pro Math \citep{wang2024mmlupro}, as reported in Table~\ref{tab:qwen35_9b_benchmark}. Appendix~\ref{app:hyperparams} gives the complete training protocol, search spaces, selected learning rates, parameter routing, and task definitions.

\begin{table}[htbp]
	\caption{Qwen3.5-9B full-parameter fine-tuning and task evaluation (three-seed mean $\pm$ sample SD). All task columns report accuracy in percent. Each optimizer is independently tuned, and downstream evaluation uses the minimum-validation-loss checkpoint from each run. Best mean per column is bold.}
	\label{tab:qwen35_9b_benchmark}
	\centering
	\normalsize
	\setlength{\tabcolsep}{2.5pt}
	\begin{tabular}{lccccc}
		\toprule
		Optimizer & Best val. loss $\downarrow$ & GSM8K $\uparrow$ & MATH-500 $\uparrow$ & MMLU-STEM $\uparrow$ & MMLU-Pro Math $\uparrow$ \\
		\midrule
		AdamW & $0.3490 \pm 0.0005$ & $79.58 \pm 0.52$ & $44.67 \pm 0.70$ & $75.53 \pm 0.36$ & $40.07 \pm 0.65$ \\
		Muon Original & $0.3497 \pm 0.0005$ & $84.23 \pm 0.80$ & $48.60 \pm 3.47$ & $75.77 \pm 0.91$ & $39.97 \pm 1.35$ \\
		Muon RMS & $0.3455 \pm 0.0005$ & $84.03 \pm 0.98$ & $53.40 \pm 1.51$ & $76.24 \pm 0.91$ & $41.77 \pm 1.98$ \\
		Adafactor & $0.3435 \pm 0.0021$ & $80.87 \pm 2.18$ & $54.33 \pm 1.94$ & $\mathbf{77.63 \pm 0.14}$ & $\mathbf{43.35 \pm 0.86}$ \\
		SOAP & $\mathbf{0.3434 \pm 0.0011}$ & $81.68 \pm 1.98$ & $48.13 \pm 0.83$ & $76.22 \pm 0.54$ & $41.40 \pm 1.04$ \\
		RODE & $0.3492 \pm 0.0022$ & $\mathbf{85.47 \pm 0.46}$ & $\mathbf{54.53 \pm 0.95}$ & $76.75 \pm 1.14$ & $41.87 \pm 1.26$ \\
		\bottomrule
	\end{tabular}
\end{table}

RODE outperforms both Muon variants on all four evaluation tasks. Across all six optimizers, it achieves the highest mean on GSM8K and MATH-500, while Adafactor leads on MMLU-STEM and MMLU-Pro Math. Appendix Figure~\ref{app:fig:qwen35_9b_diagnostics} provides the full tuning and validation-loss trajectories.

\subsection{Transfer to computer vision}

With transferred hyperparameters, RODE reaches $73.75\pm0.29$ best accuracy on CIFAR-100 and $69.74\pm0.47$ on ImageNet-1K, exceeding both Muon variants; RODE-SNR has the highest ImageNet-1K mean. RODE's final norms are $348.6\pm0.3$ and $801.8\pm27.6$, versus $1142.0\pm1.2$ and $4547.4\pm6.9$ for Muon RMS. Thus the joint accuracy and norm behavior transfers beyond language models (Figure~\ref{fig:vision_main}); Appendix~\ref{app:additional_figures} gives norm trajectories.

\begin{figure}[tbp]
	\centering
	\begin{minipage}{\panelwidth}\centering\includegraphics[width=\linewidth]{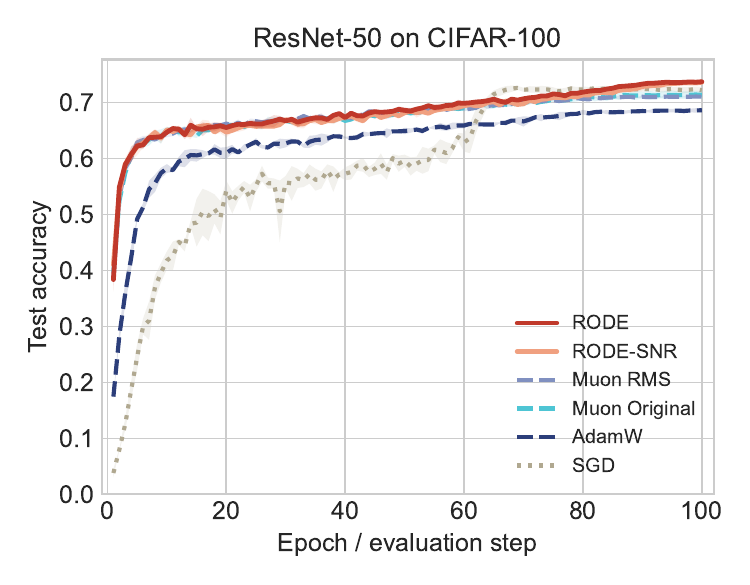}\end{minipage}\panelgap
	\begin{minipage}{\panelwidth}\centering\includegraphics[width=\linewidth]{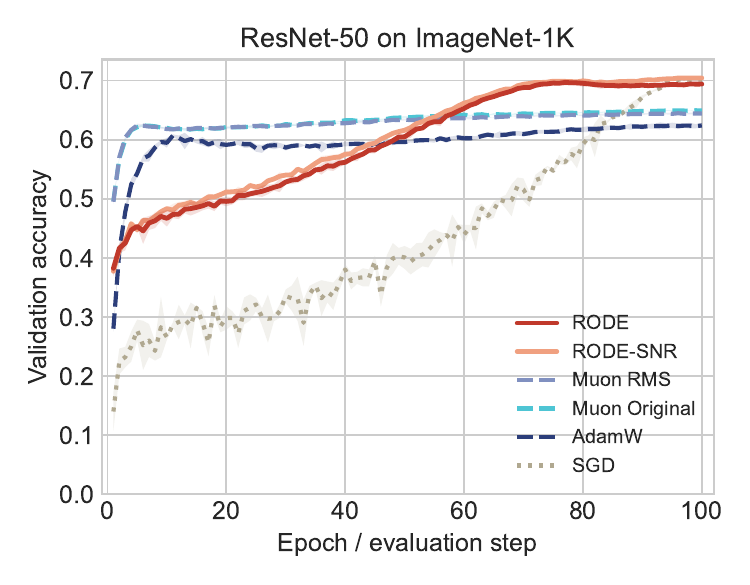}\end{minipage}
	\caption{ResNet-50 test accuracy on CIFAR-100 (left) and validation accuracy on ImageNet-1K (right), each over three seeds.}
	\label{fig:vision_main}
\end{figure}

\subsection{Component necessity}
\label{sec:ablation}

As shown in Table~\ref{tab:rode_ablation}, all six Qwen2-style-LM variants follow the same 4000-step protocol. Structural variants inherit RODE's learning rates, while RODE-SNR is independently tuned. Removing projected Newton--Schulz conditioning or directional learning substantially degrades performance, confirming that both are important components of RODE.

\begin{table}[htbp]
\caption{Three-seed 114M component study (mean $\pm$ sample SD). Structural ablations inherit RODE rates and are not retuned; RODE-SNR is an independently tuned diagnostic. $\Delta$ is relative to full RODE.}
\label{tab:rode_ablation}
\centering
\footnotesize
\setlength{\tabcolsep}{6pt}
\begin{tabular}{llcc}
\toprule
Variant & Intervention & Final val. $\downarrow$ & $\Delta$ \\
\midrule
Full RODE & Full engine & $2.7731 \pm 0.0013$ & -- \\
RODE-SNR & Add tangent-gradient ratio damping & \textbf{$2.7493 \pm 0.0011$} & $-0.0238$ \\
w/o Newton--Schulz & Remove projected NS conditioning & $3.1249 \pm 0.0113$ & $+0.3518$ \\
w/o geodesic & Use normalized retraction & $2.7736 \pm 0.0016$ & $+0.0005$ \\
RODE, radius fixed & Hold radius fixed & $2.7745 \pm 0.0017$ & $+0.0014$ \\
Radial-only control & Hold direction fixed & $3.8715 \pm 0.0023$ & $+1.0984$ \\
\bottomrule
\end{tabular}
\end{table}

\subsection{Systems cost}
\label{sec:compute}

The radial coordinate adds computation but no matrix-sized state. On the Qwen2-style-LM benchmark, RODE stores $611$ MB, equal to Muon and $33\%$ below AdamW; shape-batched fusion raises throughput from $0.59\times$ to $0.72\times$ AdamW, versus $0.77\times$ for Muon RMS. On ResNet-50, fused RODE stores $102.4$ MB and reaches $0.64\times$, versus $0.70\times$ for Muon RMS. Table~\ref{tab:compute} gives controlled measurements; Appendix~\ref{app:fused} details the implementation.

\begin{table}[htbp]
	\caption{Compute and memory benchmark. Measurements use a single RTX 5090 D v2 with \texttt{bf16} autocast: the Qwen2-style LM uses micro-batch 16 and sequence length 512, while ResNet-50 uses batch 64 and $224\times224$ inputs. Each of three runs has 8 warm-up and 25 timed steps; Step and Opt. are medians of the three per-run medians. Peak is PyTorch peak allocated memory, State is exact optimizer-state tensor storage, Opt. includes \texttt{step()} and gradient clearing, and throughput is relative to AdamW.}
	\label{tab:compute}
	\centering
	\normalsize
	\setlength{\tabcolsep}{3.5pt}
	\begin{tabular}{llccccc}
		\toprule
		Task & Optimizer & State & Peak & Step & Opt. & Throughput \\
		 & & (MB) & (GB) & (ms) & (ms) & (vs.\ AdamW) \\
		\midrule
		Qwen2-style LM & AdamW & 913.0 & 10.92 & 102.0 & 9.0 & $1.00\times$ \\
		 & Muon Original & 611.0 & 10.62 & 131.3 & 37.7 & $0.78\times$ \\
		 & Muon RMS & 611.0 & 10.62 & 132.4 & 39.1 & $0.77\times$ \\
		 & SOAP & 2643.2 & 12.67 & 137.3 & 42.1 & $0.74\times$ \\
		 & MARS-M & 913.0 & 10.96 & 158.9 & 68.1 & $0.64\times$ \\
		 & AdEMAMix & 1369.6 & 11.38 & 118.9 & 23.1 & $0.86\times$ \\
		 & RODE (loop) & 611.0 & 10.62 & 172.3 & 81.2 & $0.59\times$ \\
		 & RODE (fused) & 611.0 & 10.63 & 142.2 & 47.9 & $0.72\times$ \\
		\midrule
		ResNet-50 & AdamW & 204.5 & 3.27 & 41.3 & 2.5 & $1.00\times$ \\
		 & Muon Original & 102.4 & 3.26 & 59.0 & 20.2 & $0.70\times$ \\
		 & Muon RMS & 102.4 & 3.26 & 58.9 & 20.2 & $0.70\times$ \\
		 & SOAP & 630.3 & 3.71 & 74.0 & 34.7 & $0.56\times$ \\
		 & MARS-M & 204.5 & 3.28 & 201.6 & 161.8 & $0.20\times$ \\
		 & AdEMAMix & 306.7 & 3.37 & 58.2 & 18.6 & $0.71\times$ \\
		 & RODE (loop) & 102.4 & 3.17 & 89.5 & 50.4 & $0.46\times$ \\
		 & RODE (fused) & 102.4 & 3.17 & 64.8 & 24.7 & $0.64\times$ \\
		\bottomrule
	\end{tabular}
\end{table}

\section{Discussion}
\label{sec:discussion}

\paragraph{What RODE contributes.}
RODE combines a learned matrix-radius policy, a separately tuned directional policy, and matrix conditioning retained for direction. It turns two effects jointly determined by direct matrix addition into visible optimizer choices.

\paragraph{Why radius control matters.}
An additive step makes norm both an outcome of directional learning and a determinant of later angular progress. Growth can therefore suppress directional motion, while weight decay regulates the interaction only indirectly \citep{vanlaarhoven2017l2,hoffer2018norm,kosson2024rotational}. RODE gives the radial signal its own rule and removes the explicit $1/\rho_t$ factor from directional step size. Its goal is not to minimize norm, but to keep unintended norm growth from silently governing directional optimization.

\paragraph{What the evidence supports.}
RODE reduces observed norm growth while improving over both Muon variants in every main language-modeling and vision comparison. The GPT-2 intervention attributes gains to radial control and then the directional engine; in the 1.5B hyperparameter-transfer study, learned-radius RODE improves both loss and global norm, and fixed radius improves further. At 9B scale, each of the six optimizers is tuned with the same budget, then trained and evaluated under the same settings. RODE outperforms both Muon variants on GSM8K, MATH-500, MMLU-STEM, and MMLU-Pro Math; among all six optimizers, it attains the highest mean on GSM8K and MATH-500. Together, these results support explicit norm control as an effective mechanism in matrix optimization.

\paragraph{Limitations.}
RODE currently applies its radial--directional update only to matrix-valued parameters, while the remaining parameters are delegated to AdamW. The additional tangent projections and directional updates introduce extra computational overhead. We evaluate full-parameter fine-tuning of Qwen3.5-9B and training of a randomly initialized Qwen2-style 1.5B model, but the current study covers a limited set of language and vision architectures and training budgets; broader scales and model families remain to be explored. The theoretical convergence guarantees and rates are established under idealized regularity and alignment assumptions.

\section{Conclusion}
\label{sec:conclusion}

Direct additive matrix updates couple weight-norm evolution with directional progress, allowing norm growth to alter subsequent optimization dynamics. RODE separates these effects by treating the matrix Frobenius norm and the matrix-conditioned direction as distinct optimization coordinates with independent update rules and step sizes. Controlled interventions isolate gains from both radial control and directional learning, while experiments across language modeling and image classification show consistent improvements over Muon variants together with substantially more controlled norm growth. The 1.5B hyperparameter-transfer study further shows that these gains persist when learning rates are transferred from a smaller model, with fixed-radius RODE improving further. In Qwen3.5-9B full-parameter fine-tuning, RODE outperforms both Muon variants on all four evaluation tasks and leads the six-optimizer comparison on GSM8K and MATH-500. These results support radial--directional decoupling as an effective and practical design strategy for matrix-aware optimization.

\bibliographystyle{unsrtnat}
\FloatBarrier
\bibliography{references}

\clearpage
\appendix

\section{Complete update and fused implementation}
\label{app:update}

\subsection{Per-matrix update}

\begin{algorithm}[htbp]
	\caption{Default RODE update for one matrix parameter}
	\label{alg:rode_update}
	\footnotesize
	\centering
	\begin{tabular}{@{}c@{}}
		\toprule
		\begin{minipage}{0.96\linewidth}
		\textbf{Input:} parameter $W_t$, gradient $g_t$, raw buffer $B_{t-1}$, learning rates $\eta_\rho,\eta_{\mathrm{dir}}$, momentum $\beta$, floor $\kappa_{\mathrm{floor}}$, NS steps $K$, and safeguards $\epsilon,\epsilon_{\mathrm{NS}}$.
		\end{minipage} \\
		\midrule
		\begin{minipage}{0.96\linewidth}
		\begin{enumerate}
			\setlength{\itemsep}{2pt}
			\setlength{\parskip}{0pt}
			\item If $\|W_t\|_\fro>\epsilon$, form $\rho_t=\|W_t\|_\fro$, $U_t=W_t/\rho_t$, and $\mathcal P_t(X)=X-\langle X,U_t\rangle_\fro U_t$.
			\item Set $\widetilde B_t=\beta\mathcal P_t(B_{t-1})+g_t$.
			\item Set $s_t=\langle\widetilde B_t,U_t\rangle_\fro$ and $\rho_{t+1}=\max\{\rho_t-\eta_\rho s_t,\kappa_{\mathrm{floor}}\rho_t\}$.
			\item Set $M_t^\perp=\mathcal P_t(\widetilde B_t)=\widetilde B_t-s_tU_t$.
			\item Transpose $M_t^\perp$ and $U_t$ if needed so rows $\ge$ columns; initialize $X_0=M_t^\perp/\|M_t^\perp\|_\fro$ in the nonzero branch.
			\item For $k=0,\ldots,K-1$, set
			\[
			X_{k+1}=\mathcal P_t\!\left[3.4445X_k+X_k\{-4.7750X_k^\top X_k+2.0315(X_k^\top X_k)^2\}\right].
			\]
			Undo any transpose and set $\widehat Q_t=X_K$.
			\item For $q_t=\|\widehat Q_t\|_\fro>0$, set $D_t^\perp=\widehat Q_t/q_t$, $\theta_t=\eta_{\mathrm{dir}}q_t$, and $U_{t+1}=U_t\cos\theta_t-D_t^\perp\sin\theta_t$. Renormalize $U_{t+1}$.
			\item Write $W_{t+1}=\rho_{t+1}U_{t+1}$ and store $B_t=\widetilde B_t$.
		\end{enumerate}
		\end{minipage} \\
		\bottomrule
	\end{tabular}
\end{algorithm}

Zero parameter norms and zero projected kernels take safeguarded no-direction branches. RODE-SNR changes $\theta_t$ to $\eta_{\mathrm{dir}}\gamma_tq_t$.

\subsection{Shape-batched implementation}
\label{app:fused}

The per-parameter form of RODE iterates over matrices in Python, issuing many small kernels for the split, projection, Newton--Schulz conditioning, and geodesic step. Table~\ref{tab:compute} shows $0.59\times$ AdamW throughput on the Qwen2-style-LM configuration.

The fused implementation removes the Python loop without changing the per-matrix computation. Within a parameter group, parameters are bucketed by shape; the $B$ matrices of a common shape $m\times n$ are stacked into a single tensor $\mathcal{W}\in\R^{B\times m\times n}$, and likewise their gradients and momentum buffers. Each step of the engine then becomes a batched operation that broadcasts over the leading dimension:
\begin{itemize}
  \item \textbf{Polar split.} $\rho=\|\mathcal{W}\|_\fro$ and $\mathcal{U}=\mathcal{W}/\rho$ use a Frobenius norm reduced over the last two dimensions, giving $\rho\in\R^{B\times 1\times 1}$.
  \item \textbf{Tangent momentum.} The projection $\mathcal{P}_{\mathcal U}(\mathcal{X})=\mathcal{X}-\langle\mathcal{X},\mathcal{U}\rangle\,\mathcal{U}$ uses a batched inner product (an elementwise product summed over the matrix dimensions, kept as $B\times 1\times 1$). The fused code first forms the raw buffer $\widetilde{\mathcal M}\leftarrow\beta\mathcal{P}_{\mathcal U}(\mathcal{M}_{\mathrm{old}})+\mathcal{G}$, then removes the radial component for the directional engine:
  $\mathcal{M}^{\perp}\leftarrow\mathcal{P}_{\mathcal U}(\widetilde{\mathcal M})=\beta\mathcal{P}_{\mathcal U}(\mathcal{M}_{\mathrm{old}})+\mathcal{P}_{\mathcal U}(\mathcal{G})$. The raw buffer $\widetilde{\mathcal M}$ is what is written back; the next step begins by reprojecting it to the then-current tangent space.
  \item \textbf{Newton--Schulz.} The quintic iteration is batched matrix multiplication: $\mathcal{A}=\mathcal{X}^\top\mathcal{X}$ and $\mathcal{X}\!\leftarrow\!3.4445\,\mathcal{X}+\mathcal{X}\,(-4.7750\,\mathcal{A}+2.0315\,\mathcal{A}^2)$ are all batched \texttt{bmm}s, with the tangent reprojection applied batched after each iteration. Transposition (the $m<n$ case) is uniform within a bucket.
  \item \textbf{Geodesic step.} $\cos$ and $\sin$ of the per-matrix angle $\theta\in\R^{B\times 1\times 1}$ are applied elementwise, and the updated matrices are written back in bulk with \texttt{torch.\_foreach\_copy\_}.
\end{itemize}
Matrices of different shapes form separate buckets. Loop and batched implementations agree to \texttt{bf16} numerical tolerance ($9\times10^{-4}$ in a five-step check). Batching raises Qwen2-style-LM throughput from $0.59\times$ to $0.72\times$ AdamW (Table~\ref{tab:compute}), compared with $0.77\times$ for Muon RMS at equal state memory. RODE performs additional projection and geodesic operations in this comparison.

\section{Optimizer Definitions and Hyperparameters}
\label{app:hyperparams}

This appendix gives the exact optimizer routing, sweep budgets, and selected hyperparameters used in the reported experiments. We write $\eta_{\mathrm{2d/dir}}$ for the matrix-parameter learning rate; for RODE-family optimizers this is the angular learning rate $\eta_{\mathrm{dir}}$, while for non-RODE matrix optimizers it is the ordinary matrix-parameter learning rate. We write $\eta_\rho$ for RODE's radial learning rate and $\eta_{\mathrm{1d}}$ for the non-matrix learning rate. A dash denotes an unused parameter. Unless stated otherwise, language runs use two-dimensional non-embedding, non-head parameters as the matrix set, with all remaining parameters routed to the listed non-matrix optimizer. Vision runs route all parameters with \texttt{ndim >= 2} to the matrix optimizer; for Muon baselines, convolutional tensors are reshaped to matrices before the Muon step.

\paragraph{Hyperparameter selection protocol.}
For each sweep group in Table~\ref{tab:sweep_budget}, all optimizers in that group use the listed number of Optuna trials, training horizon, data source, model configuration, evaluation protocol, and log-uniform search spaces. The Qwen2-style sweeps evaluate once at the end of the sweep horizon and select the trial with the lowest validation loss. The standard GPT-2 sweep selects the trial with the lowest mean training loss over the final 30 steps of its 300-step protocol. After selection, the chosen values are fixed for the formal multi-seed runs. When a row is marked as transferred in Tables~\ref{tab:gpt2_hyperparams}--\ref{tab:imagenet_hyperparams}, the target task is not re-swept: the hyperparameters are copied from the stated source sweep before running the target-task experiment. CIFAR-100, ImageNet-1K, and the Qwen2-style 1.5B FineWeb experiment are transfer evaluations, not task-specific tuning studies. The separate Qwen3.5-9B sweep and training protocol is specified below and in Table~\ref{tab:qwen35_9b_hyperparams}.

\begin{table}[tbp]
	\caption{Operational baseline definitions used in the code. Matrix weight decay is the task-level \texttt{--wd}; non-matrix AdamW uses weight decay $9\times10^{-4}$ unless the row states otherwise.}
	\label{tab:optimizer_definitions}
	\centering
	\normalsize
	\setlength{\tabcolsep}{3pt}
	\begin{tabular}{@{}>{\raggedright\arraybackslash}p{0.14\linewidth}>{\raggedright\arraybackslash}p{0.26\linewidth}>{\raggedright\arraybackslash}p{0.20\linewidth}>{\raggedright\arraybackslash}p{0.30\linewidth}@{}}
		\toprule
		Optimizer & Matrix-parameter update & Non-matrix update & Fixed optimizer constants \\
		\midrule
		AdamW & AdamW on matrix parameters & AdamW on all remaining parameters & betas $(0.9,0.95)$ \\
		Muon Original & Muon, original LR scaling & AdamW & momentum $0.95$; matrix wd $0$ \\
		Muon RMS & Muon, RMS-matched LR scaling & AdamW & momentum $0.95$; matrix wd $0$ \\
		RODE & Radial update + tangent Newton--Schulz geodesic step & AdamW & momentum $0.9$; NS steps $5$; $\kappa_{\mathrm{floor}}=10^{-6}$; $\epsilon=10^{-12}$ \\
		RODE-SNR & RODE Add tangent-gradient ratio damping & AdamW & RODE constants; ratio clamp $[0.01,1]$ \\
		SOAP & SOAP on matrix parameters & AdamW & betas $(0.95,0.95)$, precondition frequency $10$, max precondition dim $10000$ \\
		MARS-M & MARS-M on matrix parameters & AdamW & momentum $0.95$, NS steps $5$, $\gamma=0.025$, approximate mode \\
		AdEMAMix & AdEMAMix on matrix parameters & AdEMAMix on all remaining parameters & betas $(0.9,0.999,0.9999)$, $\alpha=2.0$ \\
		SGD & SGD on matrix parameters & SGD on all remaining parameters & momentum $0.9$ \\
		\bottomrule
	\end{tabular}
\end{table}

\begin{table}[tbp]
	\caption{Learning-rate sweep budgets and search spaces. All sweep ranges are log-uniform. The selected values from these sweeps are reported in the task-specific hyperparameter tables below. The added baselines are swept on the Qwen2-style LM/WikiText-103 benchmark and then transferred to standard GPT-2 and CIFAR-100.}
	\label{tab:sweep_budget}
	\centering
	\normalsize
	\setlength{\tabcolsep}{4pt}
	\begin{tabular}{@{}p{0.14\linewidth}p{0.17\linewidth}p{0.12\linewidth}p{0.49\linewidth}@{}}
		\toprule
		Source & Optimizer & Budget & Search space \\
		\midrule
		GPT-2 & AdamW & $40{\times}300$ & $\eta:[10^{-4},10^{-2}]$; weight decay $[10^{-4},10^{-1}]$ \\
		GPT-2 & Muon Original & $40{\times}300$ & $\eta_{\mathrm{2d/dir}}:[10^{-4},2{\times}10^{-1}]$; $\eta_{\mathrm{1d}}:[10^{-5},10^{-2}]$ \\
		GPT-2 & Muon RMS & $40{\times}300$ & $\eta_{\mathrm{2d/dir}}:[10^{-5},5{\times}10^{-2}]$; $\eta_{\mathrm{1d}}:[10^{-5},10^{-2}]$ \\
		GPT-2 & RODE/RODE-SNR & $40{\times}300$ & $\eta_{\mathrm{2d/dir}}:[10^{-4},5{\times}10^{-1}]$; $\eta_\rho:[10^{-5},10^{-1}]$; $\eta_{\mathrm{1d}}:[10^{-5},10^{-2}]$ \\
		Qwen2-style & AdamW & $30{\times}500$ & $\eta_{\mathrm{2d/dir}},\eta_{\mathrm{1d}}:[3{\times}10^{-5},3{\times}10^{-3}]$ \\
		Qwen2-style & Muon Original/Muon RMS & $30{\times}500$ & $\eta_{\mathrm{2d/dir}}:[5{\times}10^{-3},3{\times}10^{-1}]$; $\eta_{\mathrm{1d}}:[3{\times}10^{-5},3{\times}10^{-3}]$ \\
		Qwen2-style & RODE/RODE-SNR & $30{\times}500$ & $\eta_{\mathrm{2d/dir}}:[10^{-4},5{\times}10^{-2}]$; $\eta_\rho:[10^{-6},5{\times}10^{-3}]$; $\eta_{\mathrm{1d}}:[3{\times}10^{-5},3{\times}10^{-3}]$ \\
		Qwen2-style & SOAP & $30{\times}500$ & $\eta_{\mathrm{2d/dir}}:[3{\times}10^{-5},10^{-2}]$; $\eta_{\mathrm{1d}}:[3{\times}10^{-5},3{\times}10^{-3}]$ \\
		Qwen2-style & MARS-M & $30{\times}500$ & $\eta_{\mathrm{2d/dir}}:[3{\times}10^{-5},3{\times}10^{-2}]$; $\eta_{\mathrm{1d}}:[3{\times}10^{-5},3{\times}10^{-3}]$ \\
		Qwen2-style & AdEMAMix & $30{\times}500$ & $\eta_{\mathrm{2d/dir}},\eta_{\mathrm{1d}}:[3{\times}10^{-5},3{\times}10^{-3}]$ \\
		\bottomrule
	\end{tabular}
\end{table}

\begin{table}[tbp]
	\caption{Standard GPT-2/WikiText-103 hyperparameters for the 8000-step main benchmark. Runs use sequence length 512, effective batch size 128, micro-batch size 16, \texttt{bf16} autocast, evaluation every 500 steps, 100 warmup steps, and cosine decay.}
	\label{tab:gpt2_hyperparams}
	\centering
	\normalsize
	\setlength{\tabcolsep}{3pt}
	\begin{tabular}{@{}p{0.15\linewidth}ccccc p{0.15\linewidth}@{}}
		\toprule
		Optimizer & $\eta_{\mathrm{2d/dir}}$ & $\eta_\rho$ & $\eta_{\mathrm{1d}}$ & Matrix wd & 1D wd & Source \\
		\midrule
		AdamW & $4.483{\times}10^{-4}$ & -- & $4.483{\times}10^{-4}$ & $9.077{\times}10^{-4}$ & $9.077{\times}10^{-4}$ & GPT-2 sweep \\
		Muon Original & $1.264{\times}10^{-2}$ & -- & $1.314{\times}10^{-5}$ & $0$ & $9{\times}10^{-4}$ & GPT-2 sweep \\
		Muon RMS & $1.806{\times}10^{-3}$ & -- & $9.274{\times}10^{-4}$ & $0$ & $9{\times}10^{-4}$ & GPT-2 sweep \\
		RODE & $6.432{\times}10^{-4}$ & $1.502{\times}10^{-2}$ & $5.012{\times}10^{-3}$ & $0$ & $9{\times}10^{-4}$ & GPT-2 sweep \\
		RODE-SNR & $2.413{\times}10^{-3}$ & $1.288{\times}10^{-4}$ & $2.311{\times}10^{-4}$ & $0$ & $9{\times}10^{-4}$ & GPT-2 sweep \\
		SOAP & $2.900{\times}10^{-3}$ & -- & $2.591{\times}10^{-3}$ & $10^{-2}$ & $9{\times}10^{-4}$ & Qwen transfer \\
		MARS-M & $4.097{\times}10^{-3}$ & -- & $2.693{\times}10^{-3}$ & $10^{-2}$ & $9{\times}10^{-4}$ & Qwen transfer \\
		AdEMAMix & $2.237{\times}10^{-4}$ & -- & $1.779{\times}10^{-3}$ & $10^{-2}$ & $9{\times}10^{-4}$ & Qwen transfer \\
		\bottomrule
	\end{tabular}
\end{table}

\begin{table}[tbp]
	\caption{Qwen2-style LM/WikiText-103 hyperparameters for the aligned 4000-step main benchmark and component ablation. The model has 114M parameters. All runs use sequence length 512, effective batch size 128, micro-batch size 16, \texttt{bf16} autocast, the official validation split, evaluation every 200 steps, 200 warmup steps, and cosine decay. This locally instantiated architecture uses a separate protocol from Table~\ref{tab:gpt2_hyperparams}.}
	\label{tab:qwen2small_hyperparams}
	\centering
	\normalsize
	\setlength{\tabcolsep}{3pt}
	\begin{tabular}{@{}p{0.15\linewidth}ccccc p{0.15\linewidth}@{}}
		\toprule
		Optimizer & $\eta_{\mathrm{2d/dir}}$ & $\eta_\rho$ & $\eta_{\mathrm{1d}}$ & Matrix wd & 1D wd & Source \\
		\midrule
		AdamW & $1.544{\times}10^{-4}$ & -- & $2.988{\times}10^{-3}$ & $10^{-2}$ & $9{\times}10^{-4}$ & Qwen sweep \\
		Muon Original & $2.257{\times}10^{-2}$ & -- & $2.897{\times}10^{-3}$ & $0$ & $9{\times}10^{-4}$ & Qwen sweep \\
		Muon RMS & $5.160{\times}10^{-3}$ & -- & $2.913{\times}10^{-3}$ & $0$ & $9{\times}10^{-4}$ & Qwen sweep \\
		RODE & $1.206{\times}10^{-3}$ & $2.438{\times}10^{-3}$ & $2.915{\times}10^{-3}$ & $0$ & $9{\times}10^{-4}$ & Qwen sweep \\
		RODE-SNR & $4.544{\times}10^{-3}$ & $3.698{\times}10^{-5}$ & $2.974{\times}10^{-3}$ & $0$ & $9{\times}10^{-4}$ & Qwen sweep \\
		SOAP & $2.900{\times}10^{-3}$ & -- & $2.591{\times}10^{-3}$ & $10^{-2}$ & $9{\times}10^{-4}$ & Qwen sweep \\
		MARS-M & $4.097{\times}10^{-3}$ & -- & $2.693{\times}10^{-3}$ & $10^{-2}$ & $9{\times}10^{-4}$ & Qwen sweep \\
		AdEMAMix & $2.237{\times}10^{-4}$ & -- & $1.779{\times}10^{-3}$ & $10^{-2}$ & $9{\times}10^{-4}$ & Qwen sweep \\
		\bottomrule
	\end{tabular}
\end{table}

\begin{table}[tbp]
	\caption{Qwen2-style 1.5B/FineWeb transfer hyperparameters. This three-seed experiment reuses learning rates selected on the smaller Qwen2-style LM without a new sweep. Runs use streaming FineWeb \texttt{sample-100BT}, sequence length 1024, effective batch size 128, micro-batch size 2 per GPU, \texttt{bf16} parameters/autocast, 20000 optimizer steps, evaluation every 500 steps, checkpointing every 100 steps, 200 warmup steps, cosine decay, and \texttt{torch.compile(default)}. The fixed evaluation cache is identical across optimizers.}
	\label{tab:scaling_hyperparams}
	\centering
	\normalsize
	\setlength{\tabcolsep}{4pt}
	\begin{tabular}{@{}p{0.15\linewidth}cccccc p{0.17\linewidth}@{}}
		\toprule
		Optimizer & $\eta_{\mathrm{2d/dir}}$ & $\eta_\rho$ & $\eta_{\mathrm{1d}}$ & Matrix wd & 1D wd & Accum. & Implementation \\
		\midrule
		RODE, fixed $\rho$ & $1.206{\times}10^{-3}$ & $0$ & $2.915{\times}10^{-3}$ & $0$ & $9{\times}10^{-4}$ & $64$ & fused RODE, fixed radius \\
		RODE & $1.206{\times}10^{-3}$ & $2.438{\times}10^{-3}$ & $2.915{\times}10^{-3}$ & $0$ & $9{\times}10^{-4}$ & $64$ & fused RODE, chunk 2 \\
		Muon RMS & $5.160{\times}10^{-3}$ & -- & $2.913{\times}10^{-3}$ & $0$ & $9{\times}10^{-4}$ & $64$ & PyTorch Muon RMS \\
		\bottomrule
	\end{tabular}
\end{table}

\FloatBarrier

\paragraph{Qwen3.5-9B full-parameter fine-tuning.}
All six optimizers start from the same non-Base Qwen3.5-9B checkpoint and use the same materialized GSM8K split: 7099 training examples and 374 validation examples. Each optimizer receives 20 Optuna trials of 800 steps. The Adafactor, Muon Original, and SOAP studies enqueue three, four, and four fixed candidate settings, respectively, before continuing with log-uniform sampling; these candidates count toward the same 20-trial budget. Validation is run every 100 steps on all 374 examples, and the trial with the lowest observed validation loss supplies the learning rates for the formal runs. Formal training uses three distinct seeds for 3000 steps, sequence length 1024, micro-batch size 1, gradient accumulation 8, BF16 full-parameter training, gradient checkpointing, gradient clipping at norm 1.0, and zero weight decay. The schedule applies 3\% linear warmup followed by cosine decay. Evaluation is run every 100 steps, and the checkpoint with the minimum validation loss is saved for downstream evaluation.

The parameter routing in this experiment differs from the language-model default above: every trainable two-dimensional tensor, including any embedding and output-head matrices, belongs to the matrix set. Muon Original, Muon RMS, SOAP, and RODE update this set with their matrix rule and route all remaining tensors to AdamW with betas $(0.9,0.95)$. AdamW and Adafactor each update all trainable parameters with one optimizer. Here Adafactor uses an explicit learning rate with relative-step scaling, parameter scaling, and warmup initialization disabled. SOAP uses betas $(0.9,0.95)$, precondition frequency 100, maximum precondition dimension 1024, and disables dimension merging, one-dimensional preconditioning, and gradient normalization.

For downstream evaluation, every run loads its minimum-validation-loss checkpoint. GSM8K (1319 examples) and MATH-500 (500 examples) use greedy generation with at most 512 new tokens and normalized exact-answer matching. MMLU-STEM (3153 examples) and MMLU-Pro Math (1351 examples) use a fixed zero-shot multiple-choice prompt: each option letter is scored by its summed sequence log-likelihood, and the highest-scoring option is selected. Table~\ref{tab:qwen35_9b_benchmark} reports accuracy over each complete materialized evaluation set.

\begin{table}[tbp]
	\caption{Qwen3.5-9B learning-rate searches and selected values. Every optimizer receives 20 trials of 800 steps; all searched ranges are log-uniform. ``Same'' denotes a single learning rate applied to all trainable parameters.}
	\label{tab:qwen35_9b_hyperparams}
	\centering
	\normalsize
	\setlength{\tabcolsep}{3pt}
	\begin{tabular}{@{}>{\raggedright\arraybackslash}p{0.15\linewidth}>{\raggedright\arraybackslash}p{0.40\linewidth}ccc@{}}
		\toprule
		Optimizer & Search space & Selected $\eta_{\mathrm{2d/dir}}$ & Selected $\eta_\rho$ & Selected $\eta_{\mathrm{1d}}$ \\
		\midrule
		AdamW & $\eta:[3{\times}10^{-6},8{\times}10^{-5}]$ & $2.402{\times}10^{-5}$ & -- & Same \\
		Muon Original & $\eta_{\mathrm{2d/dir}}:[10^{-4},5{\times}10^{-2}]$; $\eta_{\mathrm{1d}}:[2{\times}10^{-6},10^{-4}]$ & $2.869{\times}10^{-4}$ & -- & $5.846{\times}10^{-6}$ \\
		Muon RMS & $\eta_{\mathrm{2d/dir}}:[3{\times}10^{-6},10^{-4}]$; $\eta_{\mathrm{1d}}=2{\times}10^{-5}$ fixed & $1.971{\times}10^{-5}$ & -- & $2.000{\times}10^{-5}$ \\
		Adafactor & $\eta:[10^{-6},10^{-4}]$ & $8.000{\times}10^{-6}$ & -- & Same \\
		SOAP & $\eta_{\mathrm{2d/dir}}:[10^{-5},8{\times}10^{-3}]$; $\eta_{\mathrm{1d}}:[2{\times}10^{-6},2{\times}10^{-4}]$ & $2.118{\times}10^{-5}$ & -- & $3.020{\times}10^{-6}$ \\
		RODE & $\eta_{\mathrm{2d/dir}}:[2{\times}10^{-6},3{\times}10^{-5}]$; $\eta_\rho:[10^{-6},5{\times}10^{-5}]$; $\eta_{\mathrm{1d}}:[3{\times}10^{-6},5{\times}10^{-5}]$ & $6.259{\times}10^{-6}$ & $2.919{\times}10^{-6}$ & $5.592{\times}10^{-6}$ \\
		\bottomrule
	\end{tabular}
\end{table}

\FloatBarrier

\begin{table}[tbp]
	\caption{ResNet-50/CIFAR-100 hyperparameters. Runs use 100 epochs, image size 224, effective batch size 128, three distinct seeds, and epoch-wise cosine annealing. The implemented scheduler is \texttt{CosineAnnealingLR(T\_max=epochs)} with no warmup stage. We report best test accuracy. No CIFAR-specific sweep is used: the original suite uses standard-GPT-2-selected settings, while the added baselines use settings selected on the Qwen2-style LM.}
	\label{tab:cifar_hyperparams}
	\centering
	\normalsize
	\setlength{\tabcolsep}{3pt}
	\begin{tabular}{@{}p{0.15\linewidth}ccccc p{0.15\linewidth}@{}}
		\toprule
		Optimizer & $\eta_{\mathrm{2d/dir}}$ & $\eta_\rho$ & $\eta_{\mathrm{1d}}$ & Matrix wd & 1D wd & Source \\
		\midrule
		AdamW & $4.483{\times}10^{-4}$ & -- & $4.483{\times}10^{-4}$ & $10^{-2}$ & $9{\times}10^{-4}$ & GPT-2 transfer \\
		Muon Original & $1.264{\times}10^{-2}$ & -- & $1.314{\times}10^{-5}$ & $0$ & $9{\times}10^{-4}$ & GPT-2 transfer \\
		Muon RMS & $1.806{\times}10^{-3}$ & -- & $9.274{\times}10^{-4}$ & $0$ & $9{\times}10^{-4}$ & GPT-2 transfer \\
		RODE & $6.432{\times}10^{-4}$ & $1.502{\times}10^{-2}$ & $5.012{\times}10^{-3}$ & $0$ & $9{\times}10^{-4}$ & GPT-2 transfer \\
		RODE-SNR & $2.413{\times}10^{-3}$ & $1.288{\times}10^{-4}$ & $2.311{\times}10^{-4}$ & $0$ & $9{\times}10^{-4}$ & GPT-2 transfer \\
		SOAP & $2.900{\times}10^{-3}$ & -- & $2.591{\times}10^{-3}$ & $10^{-2}$ & $9{\times}10^{-4}$ & Qwen transfer \\
		MARS-M & $4.097{\times}10^{-3}$ & -- & $2.693{\times}10^{-3}$ & $10^{-2}$ & $9{\times}10^{-4}$ & Qwen transfer \\
		AdEMAMix & $2.237{\times}10^{-4}$ & -- & $1.779{\times}10^{-3}$ & $10^{-2}$ & $9{\times}10^{-4}$ & Qwen transfer \\
		SGD & $1.0{\times}10^{-1}$ & -- & $1.0{\times}10^{-1}$ & $5{\times}10^{-4}$ & $5{\times}10^{-4}$ & standard \\
		\bottomrule
	\end{tabular}
\end{table}

\begin{table}[tbp]
	\caption{ResNet-50/ImageNet-1K hyperparameters. Runs use 100 epochs, image size 224, effective batch size 256, three distinct seeds, and epoch-wise cosine annealing without an implemented warmup stage. Table~\ref{tab:main_results} reports, for each seed, the best validation accuracy observed over 100 epochs. The added SOAP/MARS-M/AdEMAMix baselines are not run on ImageNet-1K.}
	\label{tab:imagenet_hyperparams}
	\centering
	\normalsize
	\setlength{\tabcolsep}{4pt}
	\begin{tabular}{@{}p{0.16\linewidth}ccccc p{0.13\linewidth}@{}}
		\toprule
		Optimizer & $\eta_{\mathrm{2d/dir}}$ & $\eta_\rho$ & $\eta_{\mathrm{1d}}$ & Matrix wd & 1D wd & Source \\
		\midrule
		AdamW & $4.483{\times}10^{-4}$ & -- & $4.483{\times}10^{-4}$ & $10^{-2}$ & $9{\times}10^{-4}$ & GPT-2 transfer \\
		Muon Original & $1.264{\times}10^{-2}$ & -- & $1.314{\times}10^{-5}$ & $0$ & $9{\times}10^{-4}$ & GPT-2 transfer \\
		Muon RMS & $1.806{\times}10^{-3}$ & -- & $9.274{\times}10^{-4}$ & $0$ & $9{\times}10^{-4}$ & GPT-2 transfer \\
		RODE & $6.432{\times}10^{-4}$ & $1.502{\times}10^{-2}$ & $5.012{\times}10^{-3}$ & $0$ & $9{\times}10^{-4}$ & GPT-2 transfer \\
		RODE-SNR & $2.413{\times}10^{-3}$ & $1.288{\times}10^{-4}$ & $2.311{\times}10^{-4}$ & $0$ & $9{\times}10^{-4}$ & GPT-2 transfer \\
		SGD & $1.0{\times}10^{-1}$ & -- & $1.0{\times}10^{-1}$ & $5{\times}10^{-4}$ & $5{\times}10^{-4}$ & standard \\
		\bottomrule
	\end{tabular}
\end{table}

\begin{table}[tbp]
	\caption{Ablation protocols. Each experiment uses three distinct seeds.}
	\label{tab:ablation_hyperparams}
	\centering
	\normalsize
	\setlength{\tabcolsep}{3pt}
	\begin{tabular}{@{}>{\raggedright\arraybackslash}p{0.17\linewidth}>{\raggedright\arraybackslash}p{0.22\linewidth}>{\raggedright\arraybackslash}p{0.25\linewidth}>{\raggedright\arraybackslash}p{0.26\linewidth}@{}}
		\toprule
		Experiment & Optimizers & Shared training protocol & Hyperparameter policy \\
		\midrule
		GPT-2 norm--direction intervention & Muon RMS, Muon + NormProj, RODE with fixed radius, full RODE & standard GPT-2/WikiText-103; 8000 steps; effective batch 128; 100 warmup steps; cosine decay; three distinct seeds & Muon + NormProj inherits Muon RMS rates; the fixed-radius control inherits RODE rates and holds each matrix radius at initialization \\
		RODE component ablation & RODE, RODE-SNR, fixed-radius control, radial-only control, w/o Newton--Schulz, w/o geodesic & Qwen2-style LM/WikiText-103; 114M parameters; 4000 steps; effective batch 128; 200 warmup steps; cosine decay; three distinct seeds & RODE-SNR uses its independently swept row in Table~\ref{tab:qwen2small_hyperparams}; structural variants inherit RODE's learning rates \\
		\bottomrule
	\end{tabular}
\end{table}

\clearpage

\section{Additional Experimental Diagnostics and Curves}
\label{app:additional_figures}

This appendix collects supporting diagnostics and full trajectories that complement the main results. The same plotting code and color mapping are used as in the main figures.

\subsection{Qwen2-style LM validation trajectories}

\begin{figure}[H]
	\centering
	\begin{minipage}{\panelwidth}\centering\includegraphics[width=\linewidth]{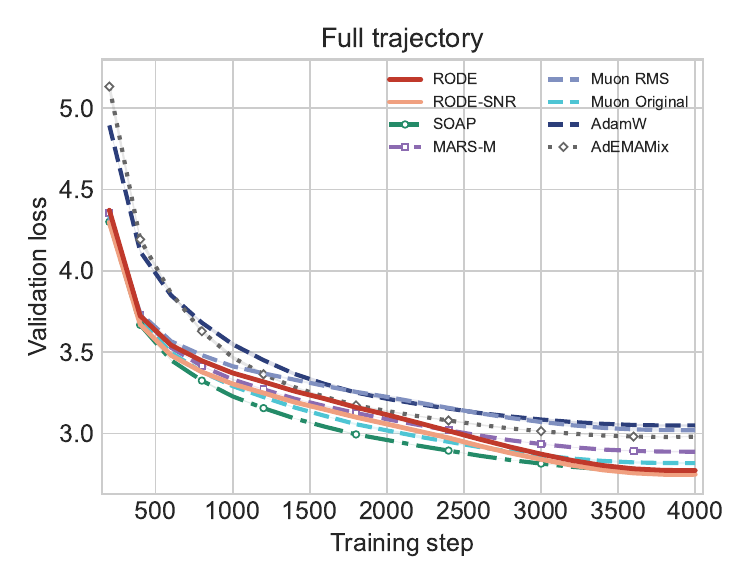}\end{minipage}\panelgap
	\begin{minipage}{\panelwidth}\centering\includegraphics[width=\linewidth]{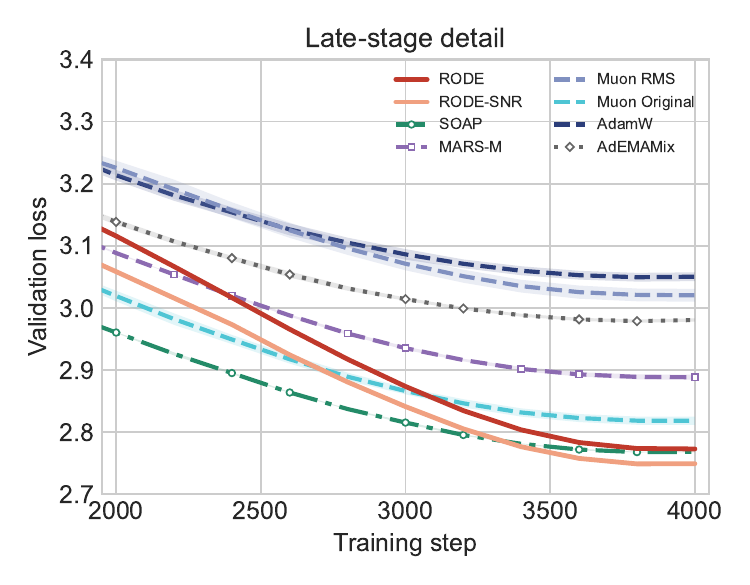}\end{minipage}
	\caption{Qwen2-style LM/WikiText-103 validation trajectories over 4000 steps (three-seed mean $\pm$ sample SD): full trajectory (left) and late-stage detail (right).}
	\label{app:fig:qwen114m}
\end{figure}

\subsection{Metric definitions}
\label{app:metric_definitions}

Diagnostics include every trainable parameter, including parameters delegated to AdamW. For the diagnostic parameter set $\mathcal P_{\mathrm{diag}}$,
\begin{align}
 \|W_t\|_{\mathrm{global}}
 &=\left(\sum_{p\in\mathcal P_{\mathrm{diag}}}\|p_t\|_\fro^2\right)^{1/2},\\
 \|\Delta W_t\|_{\mathrm{global}}
 &=\left(\sum_{p\in\mathcal P_{\mathrm{diag}}}\|p_{t+1}-p_t\|_\fro^2\right)^{1/2}.
\end{align}
The update ratio divides the latter by $\|W_t\|_{\mathrm{global}}+10^{-12}$. Gradient norm is the analogous root-sum-square immediately before the optimizer step. Only high-frequency training, gradient, and update diagnostics are EMA-smoothed for visualization; validation, accuracy, weight norm, schedules, and tabulated values use raw logs.

\subsection{1.5B scale-transfer protocol}
\label{app:scaling_results}

The 1.5B runs use sequence length 1024, effective batch size 128, 200 warmup steps, cosine decay, and learning rates selected on the smaller Qwen2-style LM. Fixed-radius RODE uses the same directional and delegated-parameter rates as full RODE, sets $\eta_\rho=0$, and holds each managed matrix's radius at initialization. All three methods use the same three distinct seeds for 20000 steps. The fixed evaluation cache contains 1024 length-1024 sequences built from FineWeb \texttt{sample-10BT}; document-level exclusion from the \texttt{sample-100BT} training stream is not enforced \citep{finewebdatasetcard}. Every optimizer is evaluated on the identical cache. Figure~\ref{fig:scaling_main} reports cache loss and global norm; Figure~\ref{app:fig:scaling_diagnostics} reports dense training diagnostics. Raw per-seed scalars and provenance are archived with the paper.

\begin{figure}[tbp]
	\centering
	\begin{minipage}{\panelwidth}\centering\includegraphics[width=\linewidth]{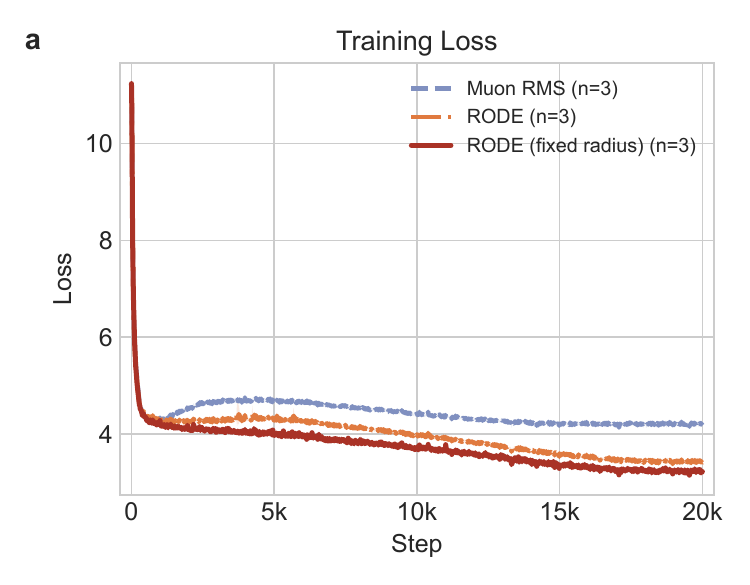}\end{minipage}\panelgap
	\begin{minipage}{\panelwidth}\centering\includegraphics[width=\linewidth]{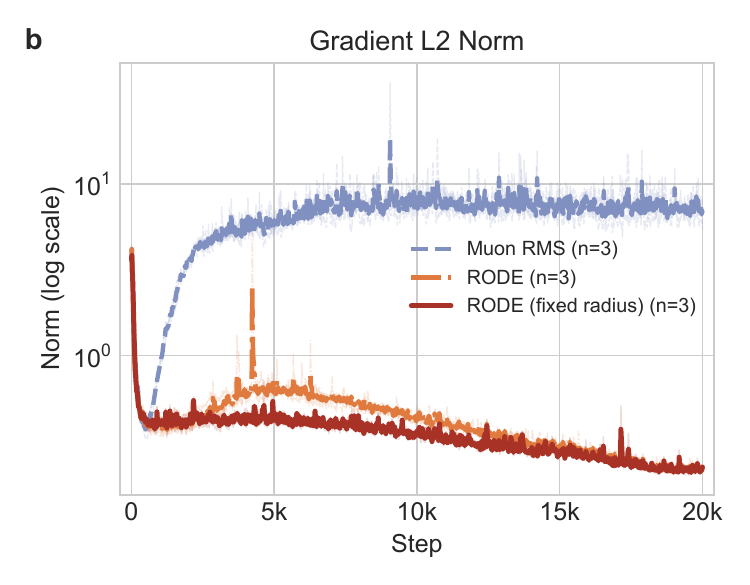}\end{minipage}
	\caption{Qwen2-style 1.5B/FineWeb training diagnostics over 20000 steps (three seeds per method): EMA-smoothed training loss (left) and gradient norm (right). Gradient norm uses a logarithmic axis.}
	\label{app:fig:scaling_diagnostics}
\end{figure}

\subsection{Qwen3.5-9B tuning and training diagnostics}

\begin{figure}[H]
	\centering
	\begin{minipage}{\panelwidth}\centering\includegraphics[width=\linewidth]{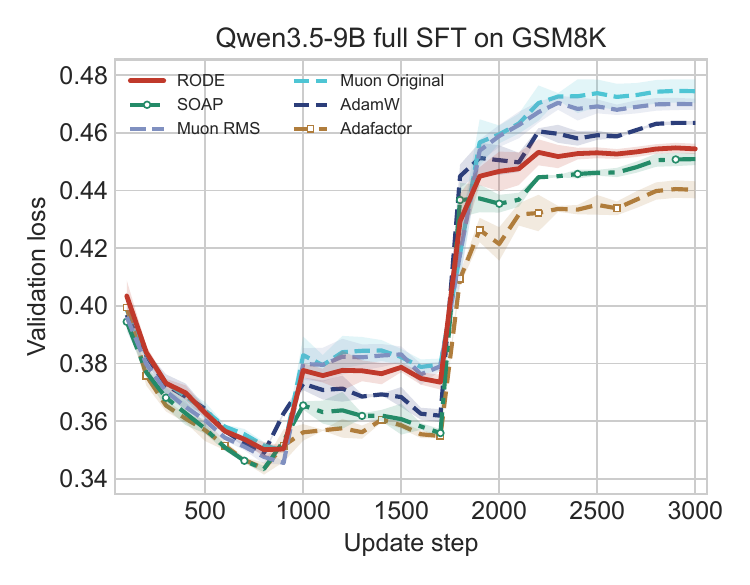}\end{minipage}\panelgap
	\begin{minipage}{\panelwidth}\centering\includegraphics[width=\linewidth]{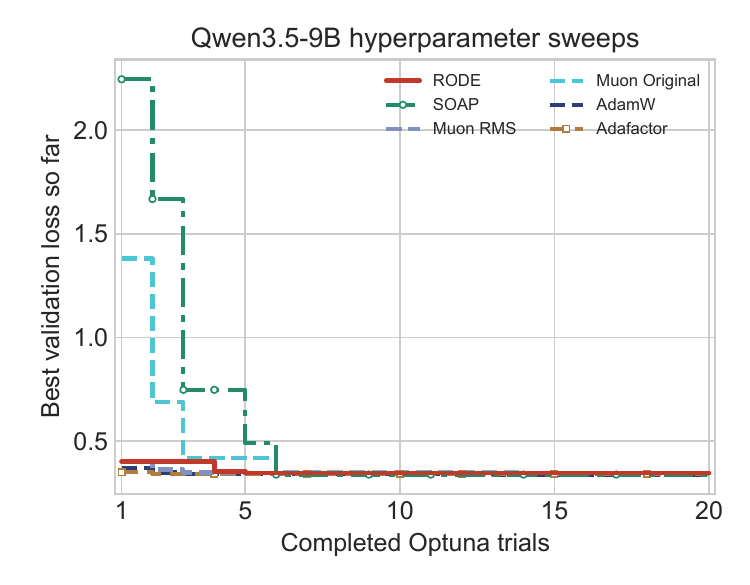}\end{minipage}
	\caption{Qwen3.5-9B full-parameter fine-tuning diagnostics for the six optimizers in Table~\ref{tab:qwen35_9b_benchmark}. Left: validation loss over the 3000-step formal runs (three-seed mean $\pm$ sample SD). Right: cumulative best objective over the 20 800-step Optuna trials for each optimizer; each trial's objective is its minimum validation loss. Both panels use raw, unsmoothed values.}
	\label{app:fig:qwen35_9b_diagnostics}
\end{figure}

\paragraph{GPT-2/WikiText-103 dynamics.}
Figure~\ref{app:fig:gpt2_supp} shows the training loss, update step size, and gradient norm for the GPT-2 runs; the GPT-2 weight norm is in the main text (Figure~\ref{fig:gpt2_main}).
\begin{figure}[tbp]
	\centering
	\begin{minipage}{\panelwidth}\centering\includegraphics[width=\linewidth]{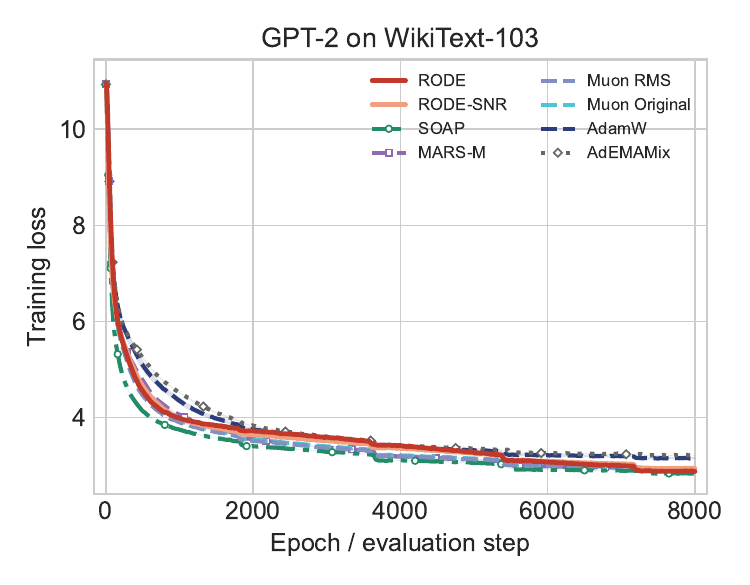}\end{minipage}\panelgap
	\begin{minipage}{\panelwidth}\centering\includegraphics[width=\linewidth]{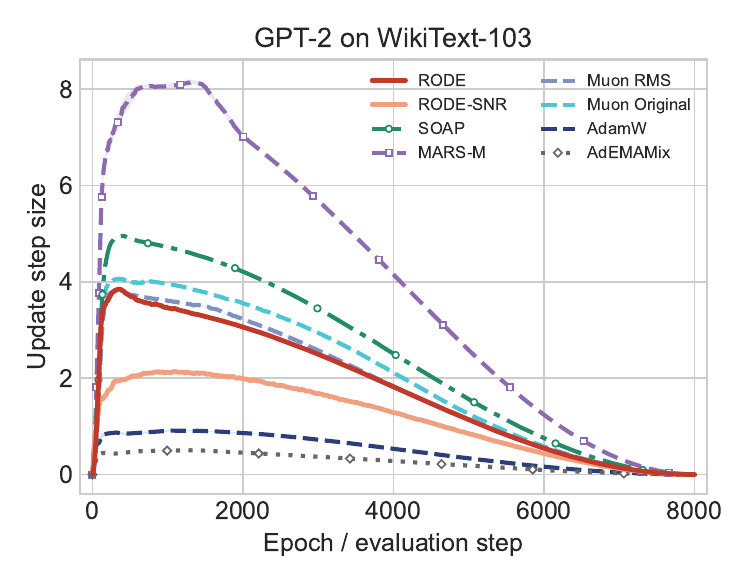}\end{minipage}\\[0.6em]
	\begin{minipage}{\panelwidth}\centering\includegraphics[width=\linewidth]{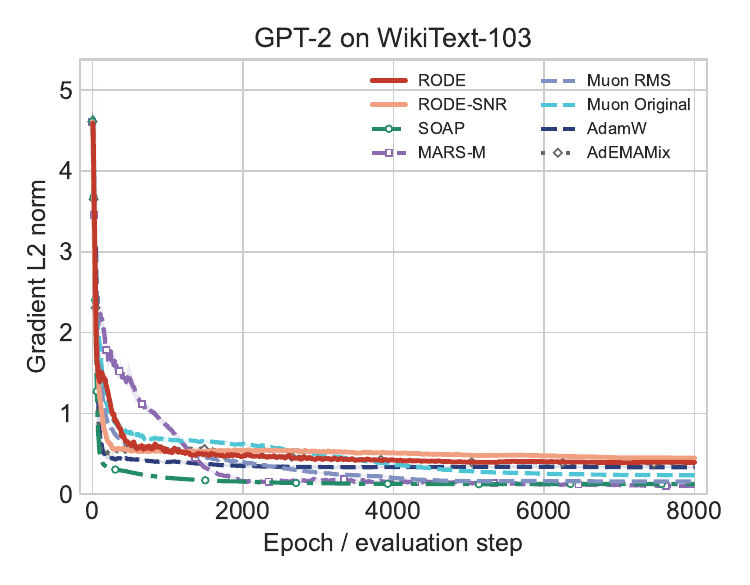}\end{minipage}
	\caption{GPT-2/WikiText-103 training dynamics (three seeds): training loss, update step size, and gradient norm.}
	\label{app:fig:gpt2_supp}
\end{figure}

\paragraph{ResNet-50 evaluation loss and weight norm.}
Figures~\ref{app:fig:vision_dyn} and~\ref{fig:vision_norm} show test loss for CIFAR-100, validation loss for ImageNet-1K, and observed full-model weight norm for both datasets.
\begin{figure}[tbp]
	\centering
	\begin{minipage}{\panelwidth}\centering\includegraphics[width=\linewidth]{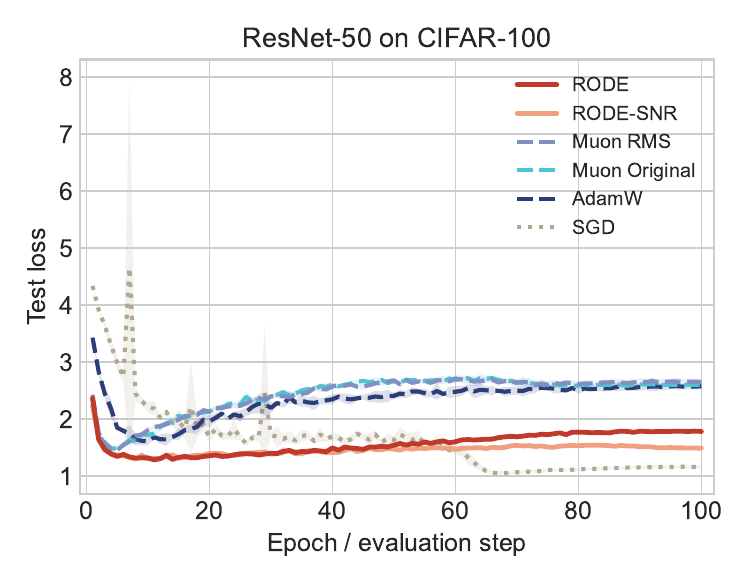}\end{minipage}\panelgap
	\begin{minipage}{\panelwidth}\centering\includegraphics[width=\linewidth]{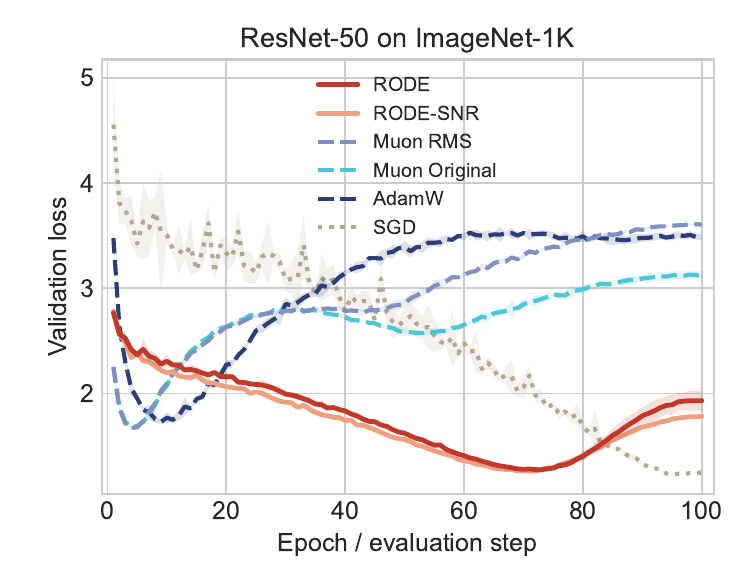}\end{minipage}
	\caption{ResNet-50 test loss on CIFAR-100 (left) and validation loss on ImageNet-1K (right), shown as three-seed mean $\pm$ sample SD.}
	\label{app:fig:vision_dyn}
\end{figure}

\begin{figure}[tbp]
	\centering
	\begin{minipage}{\panelwidth}\centering\includegraphics[width=\linewidth]{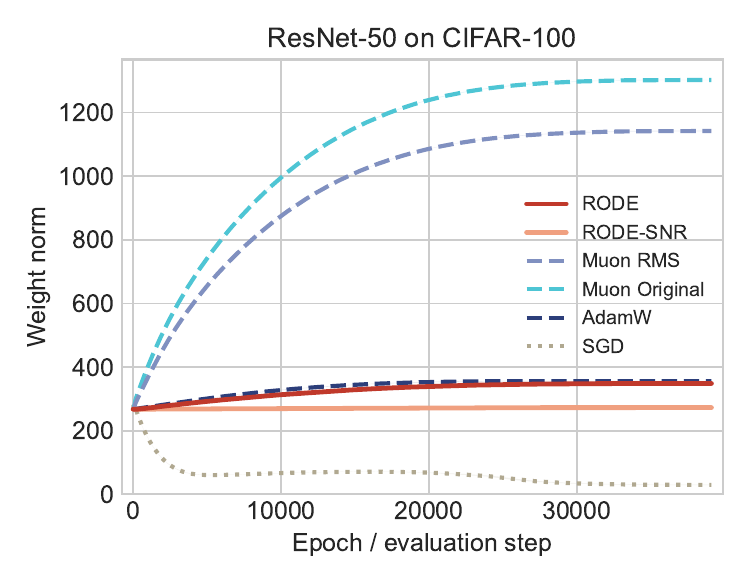}\end{minipage}\panelgap
	\begin{minipage}{\panelwidth}\centering\includegraphics[width=\linewidth]{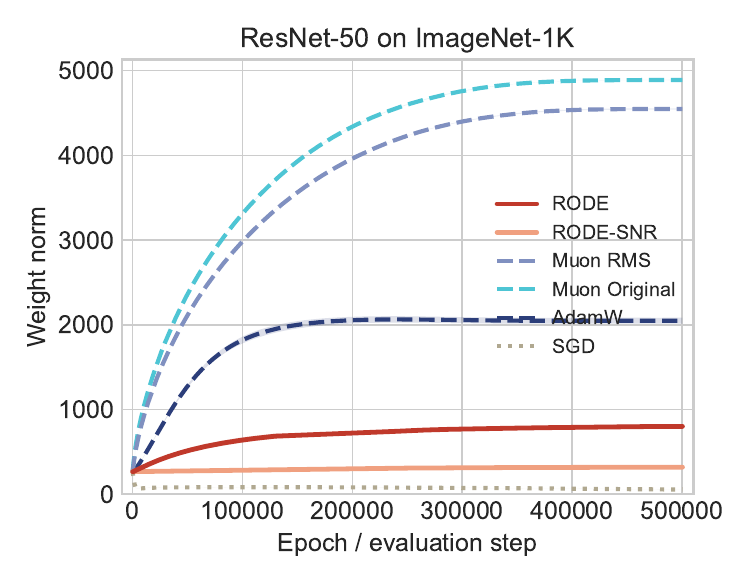}\end{minipage}
	\caption{Observed full-model weight norms on CIFAR-100 (left) and ImageNet-1K (right), shown as three-seed mean $\pm$ sample SD. RODE directly updates managed-matrix radii; these curves additionally include parameters delegated to AdamW.}
	\label{fig:vision_norm}
\end{figure}

\begin{figure}[tbp]
	\centering
	\begin{minipage}{\panelwidth}
		\centering
		\includegraphics[width=\linewidth]{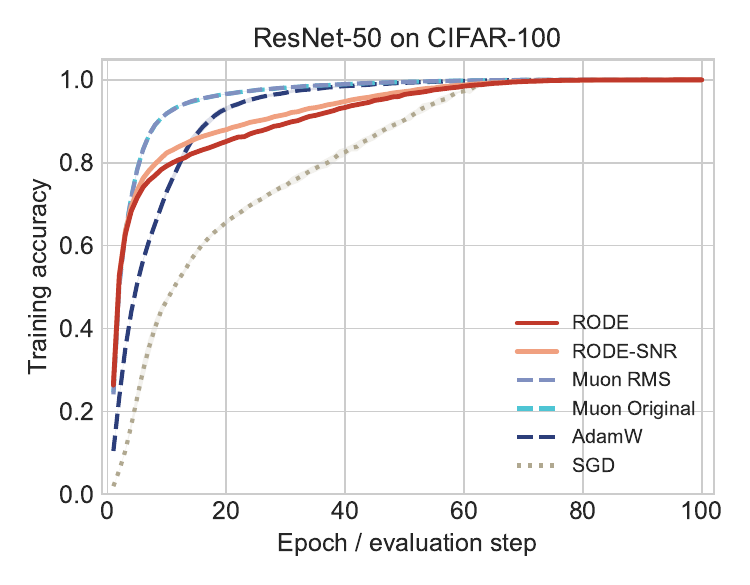}
	\end{minipage}
	\panelgap
	\begin{minipage}{\panelwidth}
		\centering
		\includegraphics[width=\linewidth]{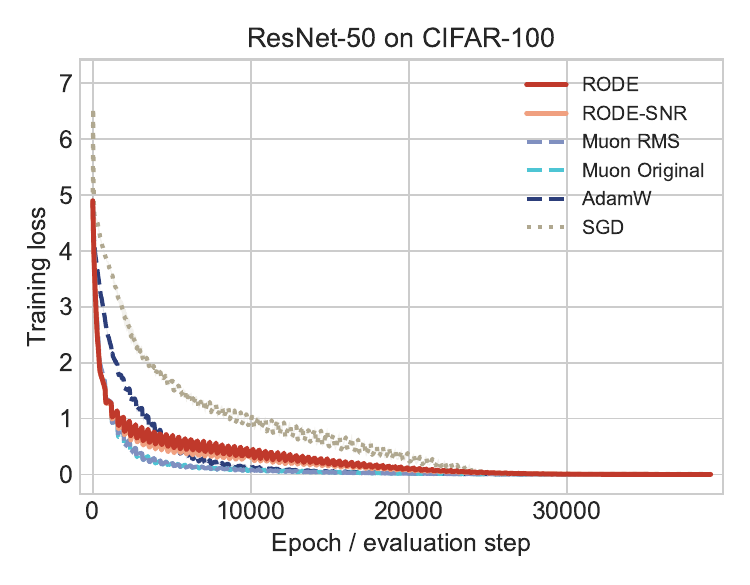}
	\end{minipage}\\[0.6em]
	\begin{minipage}{\panelwidth}
		\centering
		\includegraphics[width=\linewidth]{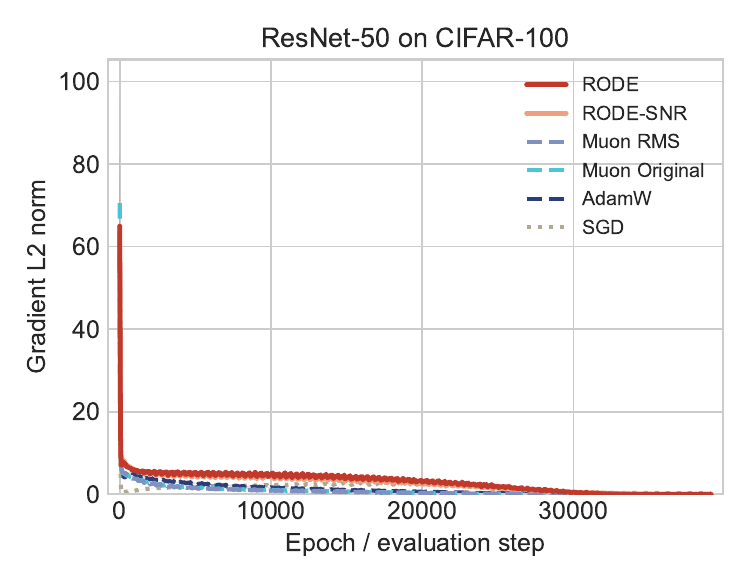}
	\end{minipage}
	\panelgap
	\begin{minipage}{\panelwidth}
		\centering
		\includegraphics[width=\linewidth]{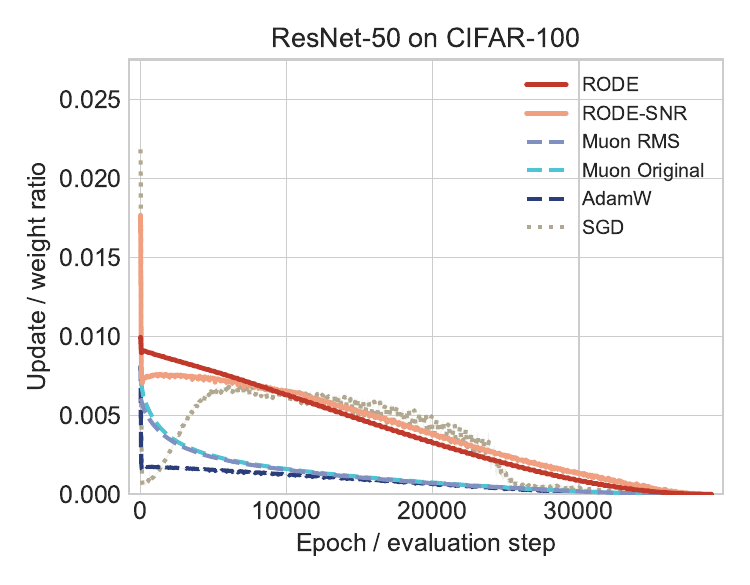}
	\end{minipage}
	\caption{Additional ResNet-50/CIFAR-100 curves: training accuracy, training loss, gradient norm, and update-to-weight ratio.}
	\label{app:fig:cifar_supp}
\end{figure}

\begin{figure}[tbp]
	\centering
	\begin{minipage}{\panelwidth}
		\centering
		\includegraphics[width=\linewidth]{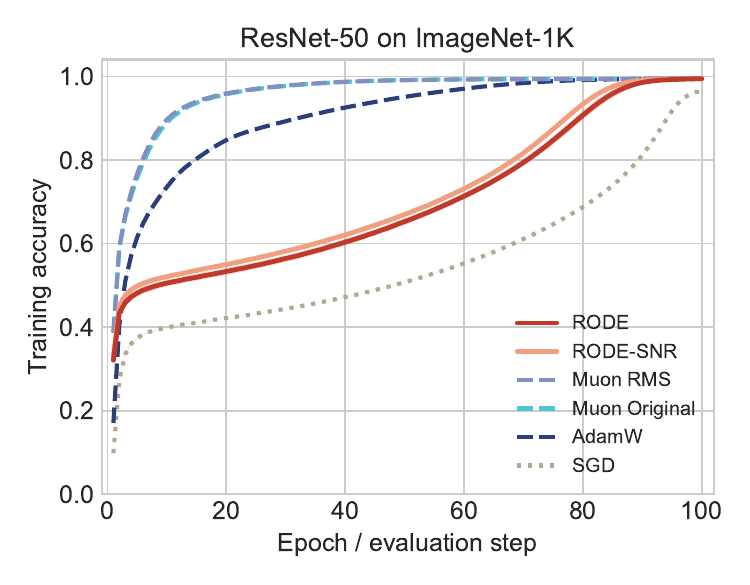}
	\end{minipage}
	\panelgap
	\begin{minipage}{\panelwidth}
		\centering
		\includegraphics[width=\linewidth]{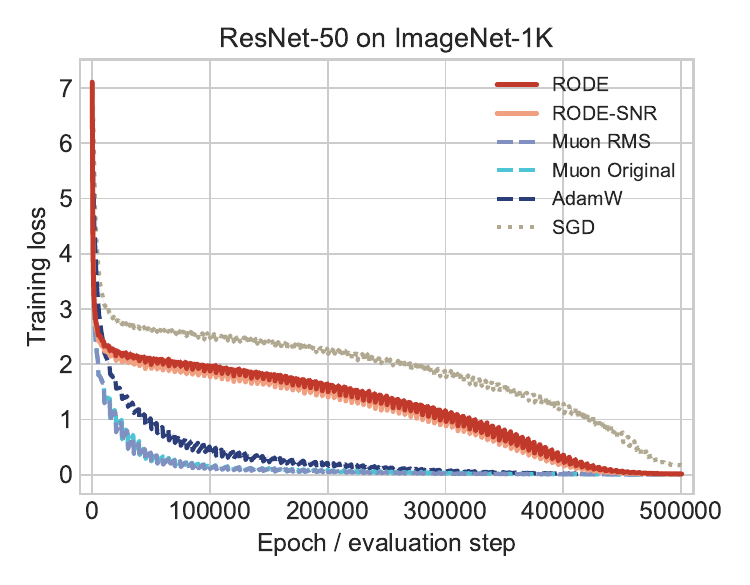}
	\end{minipage}\\[0.6em]
	\begin{minipage}{\panelwidth}
		\centering
		\includegraphics[width=\linewidth]{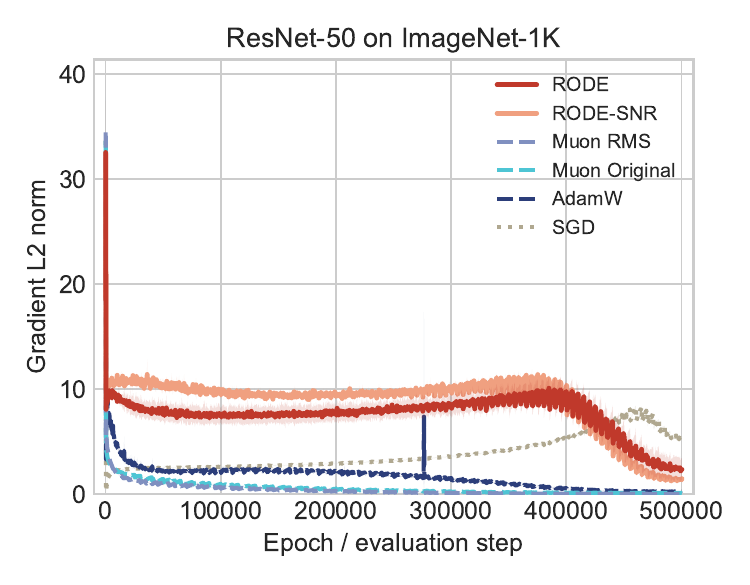}
	\end{minipage}
	\panelgap
	\begin{minipage}{\panelwidth}
		\centering
		\includegraphics[width=\linewidth]{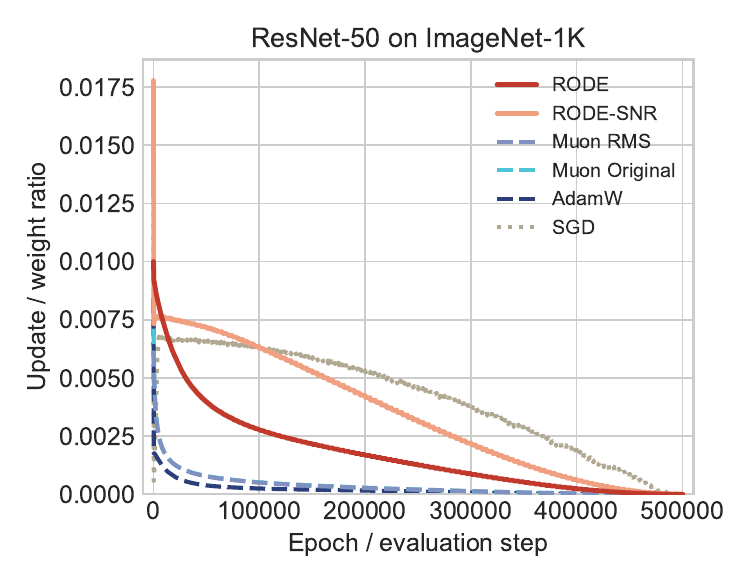}
	\end{minipage}
	\caption{Additional ResNet-50/ImageNet-1K curves from the exported per-step diagnostics: three-seed means with sample-standard-deviation bands for training accuracy, training loss, gradient norm, and update-to-weight ratio. Main ImageNet validation results are reported in Table~\ref{tab:main_results} and Figure~\ref{fig:vision_main}.}
	\label{app:fig:imagenet_supp}
\end{figure}

\clearpage

\section{Detailed Global Convergence Proof}
\label{app:global_convergence}

This appendix expands the convergence-neighborhood statement in Section~\ref{sec:theory}. We derive the geodesic displacement, split radial and tangent descent, bound the tangent tracking error, and combine the potential and kinetic inequalities with a Lyapunov function.

\subsection{Notation and algorithmic equations}

Let $f(W): \R^{m\times n}\rightarrow \R$ be the objective for a matrix parameter. For clarity assume $m\ge n$; if $m<n$, the implementation transposes the matrix before applying Newton--Schulz and transposes the result back. The Frobenius inner product is
\begin{equation}
    \langle A,B\rangle_\fro = \tr(A^\top B),
    \qquad
    \|A\|_\fro = \sqrt{\langle A,A\rangle_\fro}.
\end{equation}
At iteration $t$, RODE writes
\begin{equation}
    W_t = \rho_t U_t,
    \qquad
    \rho_t = \|W_t\|_\fro,
    \qquad
    U_t = \frac{W_t}{\rho_t},
    \qquad
    \|U_t\|_\fro = 1 .
\end{equation}
Let $\mathcal F_t$ denote the history immediately before sampling the stochastic gradient at step $t$, and write $\E_t[\cdot]=\E[\cdot\mid\mathcal F_t]$. Let $G_t = \nabla f(W_t)$ denote the full gradient and let $g_t$ denote a stochastic gradient satisfying $\E_t[g_t]=G_t$. We decompose both into radial and tangent components:
\begin{align}
    G_t^\rho &= \langle G_t,U_t\rangle_\fro,
    &
    G_t^\perp &= G_t - G_t^\rho U_t,\\
    g_t^\rho &= \langle g_t,U_t\rangle_\fro,
    &
    g_t^\perp &= g_t - g_t^\rho U_t.
\end{align}
Thus $W_t$, $U_t$, $\rho_t$, and $G_t$ are fixed under $\E_t$. Unsubscripted $\E$ denotes total expectation; below we pass from conditional one-step inequalities to total expectations using the tower property.
The radial update is
\begin{equation}
    \rho_{t+1} = \rho_t - \eta_\rho g_t^\rho .
\end{equation}
The implementation uses the guarded scalar update
\begin{equation}
    \rho_{t+1}=\max\{\rho_t-\eta_\rho g_t^\rho,\kappa_{\mathrm{floor}}\rho_t\}.
\end{equation}
The proof below analyzes the interior regime in which this floor is inactive. If the floor activates, the radial channel becomes a projected one-dimensional step and the radial descent inequality must be replaced by the corresponding projected-gradient inequality; we do not claim that branch in this appendix.
The implemented optimizer stores a raw buffer rather than only a tangent buffer:
\begin{equation}
    B_t=\beta\mathcal P_t(B_{t-1})+g_t,
    \qquad
    M_{\mathrm{raw},t}^\perp=\mathcal P_t(B_t).
\end{equation}
At the current step this gives the same tangent direction as projecting the newly formed buffer. Across steps, however, the radial part retained in $B_t$ can re-enter the next tangent space through $\mathcal P_{t+1}(B_t)$. Write
\begin{equation}
    B_t=M_{\mathrm{raw},t}^\perp+a_tU_t,
    \qquad
    a_t=\langle B_t,U_t\rangle_\fro .
\end{equation}
Then the next projected raw tangent buffer obeys
\begin{align}
    M_{\mathrm{raw},t+1}^\perp
    &=
    \mathcal P_{t+1}(B_{t+1})\nonumber\\
    &=
    \beta\mathcal P_{t+1}(M_{\mathrm{raw},t}^\perp)
    +
    g_{t+1}^\perp
    +
    r_{t+1},
\end{align}
with the cross-step radial leakage residual
\begin{equation}
    r_{t+1}=\beta a_t\mathcal P_{t+1}(U_t).
\end{equation}
The implementation stores the unnormalized momentum above. For the tracking argument we analyze its normalized counterpart
\begin{equation}
    M_t^\perp=(1-\beta)M_{\mathrm{raw},t}^\perp,
    \qquad
    \bar r_{t+1}=(1-\beta)r_{t+1},
\end{equation}
so that
\begin{equation}
    M_{t+1}^\perp=
    \beta\mathcal P_{t+1}(M_t^\perp)
    +(1-\beta)g_{t+1}^\perp
    +\bar r_{t+1}.
\end{equation}
This rescaling does not change the Newton--Schulz direction because the engine starts from $M/\|M\|_\fro$. The implemented raw buffer therefore induces an additional cross-step leakage residual; below we bound its squared contribution and absorb it into the curvature-order residual. From here on, $M_t^\perp$ denotes the normalized analysis momentum, and the appendix analyzes the nondegenerate, inactive-floor raw-buffer branch rather than the zero-norm safeguard or active-floor branches.
The projected Newton--Schulz engine starts from
\begin{equation}
    X_0 = \frac{M_t^\perp}{\|M_t^\perp\|_\fro}
\end{equation}
in the nondegenerate branch and, matching the default implementation, applies tangent reprojection after each quintic polynomial step:
\begin{align}
    Y_{k+1}
    &=
    3.4445X_k
    +
    X_k\left(-4.7750X_k^\top X_k+2.0315(X_k^\top X_k)^2\right),\\
    X_{k+1}
    &=
    Y_{k+1}-\langle Y_{k+1},U_t\rangle_\fro U_t,
    \qquad k=0,\ldots,K-1 .
\end{align}
After $K$ projected steps, define
\begin{equation}
    \widehat Q_t = X_K .
\end{equation}
Define its actual post-projection norm
\begin{equation}
    q_t=\|\widehat Q_t\|_\fro .
\end{equation}
Projection removes the component parallel to $U_t$, so $q_t$ is generally different from $\sqrt n$. The proof assumes that $\widehat Q_t$ retains a constant fraction of the polar-target alignment with $M_t^\perp$ and has nondegenerate projected norm, so that for some $c_{\mathrm{NS}}\in(0,1]$,
\begin{equation}
    \langle M_t^\perp,\widehat Q_t\rangle_\fro \ge c_{\mathrm{NS}}\|M_t^\perp\|_*,
    \qquad
    0<q_{\min}\le q_t\le q_{\max}<\infty .
\end{equation}
Define the unit tangent direction
\begin{equation}
    D_t^\perp = \frac{\widehat Q_t}{q_t},
    \qquad
    \langle D_t^\perp,U_t\rangle_\fro = 0 .
\end{equation}
Default RODE uses no scalar damping, i.e. $\gamma_t=1$ and $\gamma_{\min}=1$. The proof also covers any optional scalar damping satisfying $0<\gamma_{\min}\le\gamma_t\le1$, including the RODE-SNR ablation. The geodesic angle is
\begin{equation}
    \theta_t = \eta_{\mathrm{dir}}\gamma_t q_t .
\end{equation}
Finally,
\begin{align}
    U_{t+1} &= U_t\cos\theta_t - D_t^\perp\sin\theta_t,\\
    W_{t+1} &= \rho_{t+1}U_{t+1}.
\end{align}

\subsection{Assumptions}

\begin{assumption}[Smoothness]
\label{app:assump:smooth}
For all matrices $X,Y$,
\begin{equation}
    f(Y)\le f(X)+\langle \nabla f(X),Y-X\rangle_\fro
    +\frac{L}{2}\|Y-X\|_\fro^2 .
\end{equation}
\end{assumption}

\begin{assumption}[Unbiased noise with bounded variance]
\label{app:assump:noise}
The stochastic gradient is unbiased, and the radial and tangent noises are bounded:
\begin{align}
    \E_t[g_t^\rho-G_t^\rho] &= 0,
    &
    \E_t[(g_t^\rho-G_t^\rho)^2] &\le \sigma_\rho^2,\\
    \E_t[g_t^\perp-G_t^\perp] &= 0,
    &
    \E_t[\|g_t^\perp-G_t^\perp\|_\fro^2] &\le \sigma_\perp^2 .
\end{align}
\end{assumption}

\begin{assumption}[Bounded physical states]
\label{app:assump:bounded}
There exist constants such that
\begin{equation}
    0<\rho_{\min}\le \rho_t\le \rho_{\max},\quad
    |G_t^\rho|\le G_{\max},\quad
    \|M_t^\perp\|_\fro\le M_{\max},\quad
    |a_t|\le A_{\max}
    \quad\text{for some }M_{\max},A_{\max}>0.
\end{equation}
Here $a_t=\langle B_t,U_t\rangle_\fro$ is the radial component retained in the implemented raw buffer.
\end{assumption}

\begin{assumption}[Local tangent-gradient regularity]
\label{app:assump:tangent_lip}
Let $\mathcal P_{t+1}(X)=X-\langle X,U_{t+1}\rangle_\fro U_{t+1}$. For the small geodesic moves considered here,
\begin{equation}
    \E\left[\|\mathcal P_{t+1}(G_t^\perp)-G_{t+1}^\perp\|_\fro^2\right]
    \le
    L^2\E[\|\Delta W_t\|_\fro^2],
\end{equation}
where $\Delta W_t=W_{t+1}-W_t$.
\end{assumption}

\begin{assumption}[Aligned projected Newton--Schulz kernel]
\label{app:assump:polar}
After the finite projected Newton--Schulz iterations, the constant-step engine returns a tangent kernel satisfying, for some $c_{\mathrm{NS}}\in(0,1]$,
\begin{equation}
    \langle M_t^\perp,\widehat Q_t\rangle_\fro \ge c_{\mathrm{NS}}\|M_t^\perp\|_*,
    \qquad
    0<q_{\min}\le\|\widehat Q_t\|_\fro\le q_{\max}<\infty .
\end{equation}
This assumption isolates the optimizer dynamics from finite-iteration polynomial approximation errors, projection-induced alignment loss, and reduced-precision effects, and excludes degenerate steps where tangent projection nearly annihilates the conditioned kernel. The exact polar identity is recovered in the ideal case $c_{\mathrm{NS}}=1$.
\end{assumption}

\begin{assumption}[Bounded geodesic angle]
\label{app:assump:angle}
There exists $\theta_{\max}\in(0,\pi/4]$ such that
\begin{equation}
    \eta_{\mathrm{dir}}\le \frac{\theta_{\max}}{q_{\max}}.
\end{equation}
Since $\gamma_t\le 1$, this gives $\theta_t\le \theta_{\max}$ and ensures the global lower bound
\begin{equation}
    \sin(\theta_t)\ge k\theta_t,
    \qquad
    k=\frac{\sin(\theta_{\max})}{\theta_{\max}}
    \ge \frac{2\sqrt 2}{\pi}.
\end{equation}
\end{assumption}

\subsection{Raw-buffer leakage bound}

The geodesic update gives
\begin{equation}
    U_{t+1}=U_t\cos\theta_t-D_t^\perp\sin\theta_t,
    \qquad
    \langle U_t,D_t^\perp\rangle_\fro=0 .
\end{equation}
Hence $\langle U_t,U_{t+1}\rangle_\fro=\cos\theta_t$ and
\begin{equation}
    \|\mathcal P_{t+1}(U_t)\|_\fro^2
    =
    1-\langle U_t,U_{t+1}\rangle_\fro^2
    =
    \sin^2\theta_t .
\end{equation}
Using $\theta_t=\eta_{\mathrm{dir}}\gamma_tq_t$, $\gamma_t\le1$, and $q_t\le q_{\max}$,
\begin{equation}
    \|\mathcal P_{t+1}(U_t)\|_\fro
    \le
    \theta_t
    \le
    \eta_{\mathrm{dir}}q_{\max}.
\end{equation}
Therefore the normalized raw-buffer leakage residual satisfies
\begin{equation}
\label{app:eq:leakage_bound}
    \|\bar r_{t+1}\|_\fro^2
    \le
    (1-\beta)^2\beta^2A_{\max}^2q_{\max}^2\eta_{\mathrm{dir}}^2
    =:
    K_{\mathrm{leak}}\eta_{\mathrm{dir}}^2 .
\end{equation}

\subsection{Exact displacement expansion and second-order bound}

The physical displacement is
\begin{equation}
    \Delta W_t=W_{t+1}-W_t
    =
    \rho_{t+1}(U_t\cos\theta_t-D_t^\perp\sin\theta_t)-\rho_tU_t .
\end{equation}
Collecting the orthogonal basis terms $U_t$ and $D_t^\perp$ gives
\begin{equation}
    \Delta W_t
    =
    (\rho_{t+1}\cos\theta_t-\rho_t)U_t
    -
    \rho_{t+1}\sin\theta_t D_t^\perp .
\end{equation}
Because $\langle U_t,D_t^\perp\rangle_\fro=0$ and both directions have unit Frobenius norm,
\begin{align}
    \|\Delta W_t\|_\fro^2
    &=
    (\rho_{t+1}\cos\theta_t-\rho_t)^2
    +
    \rho_{t+1}^2\sin^2\theta_t\\
    &=
    \rho_{t+1}^2\cos^2\theta_t
    -2\rho_{t+1}\rho_t\cos\theta_t
    +\rho_t^2
    +\rho_{t+1}^2\sin^2\theta_t\\
    &=
    \rho_{t+1}^2+\rho_t^2-2\rho_{t+1}\rho_t\cos\theta_t\\
    &=
    (\rho_{t+1}-\rho_t)^2
    +
    2\rho_{t+1}\rho_t(1-\cos\theta_t).
\end{align}
Using $1-\cos\theta\le \theta^2/2$,
\begin{equation}
    \|\Delta W_t\|_\fro^2
    \le
    (\rho_{t+1}-\rho_t)^2+\rho_{t+1}\rho_t\theta_t^2 .
\end{equation}
Substituting $\rho_{t+1}-\rho_t=-\eta_\rho g_t^\rho$ and $\theta_t=\eta_{\mathrm{dir}}\gamma_t q_t$,
\begin{equation}
    \|\Delta W_t\|_\fro^2
    \le
    \eta_\rho^2(g_t^\rho)^2
    +
    \rho_{t+1}\rho_t\eta_{\mathrm{dir}}^2\gamma_t^2 q_t^2 .
\end{equation}
By Assumptions~\ref{app:assump:bounded}--\ref{app:assump:polar}, $\gamma_t\le 1$, and $q_t\le q_{\max}$,
\begin{equation}
    \|\Delta W_t\|_\fro^2
    \le
    \eta_\rho^2(g_t^\rho)^2
    +
    \rho_{\max}^2 q_{\max}^2\eta_{\mathrm{dir}}^2 .
\end{equation}
Taking conditional expectation and using
\begin{equation}
    \E_t[(g_t^\rho)^2]
    =
    (\E_t[g_t^\rho])^2+\operatorname{Var}_t(g_t^\rho)
    \le
    (G_t^\rho)^2+\sigma_\rho^2,
\end{equation}
we obtain the second-order smoothness contribution
\begin{equation}
\label{app:eq:second_order_bound}
    \frac{L}{2}\E_t[\|\Delta W_t\|_\fro^2]
    \le
    \frac{L}{2}\eta_\rho^2(G_t^\rho)^2
    +
    \frac{L}{2}\eta_\rho^2\sigma_\rho^2
    +
    \frac{L\rho_{\max}^2q_{\max}^2}{2}\eta_{\mathrm{dir}}^2 .
\end{equation}

\subsection{First-order descent: radial part}

The first-order term in smoothness is
\begin{equation}
    \langle G_t,\Delta W_t\rangle_\fro .
\end{equation}
Using $G_t=G_t^\rho U_t+G_t^\perp$ and the displacement expansion,
\begin{equation}
\label{app:eq:first_order_split}
    \langle G_t,\Delta W_t\rangle_\fro
    =
    G_t^\rho(\rho_{t+1}\cos\theta_t-\rho_t)
    -
    \rho_{t+1}\sin\theta_t\langle G_t^\perp,D_t^\perp\rangle_\fro .
\end{equation}
For the radial component,
\begin{align}
    I_\rho
    &=
    G_t^\rho(\rho_{t+1}\cos\theta_t-\rho_t)\\
    &=
    G_t^\rho(\rho_{t+1}-\rho_t)
    +
    G_t^\rho\rho_{t+1}(\cos\theta_t-1)\\
    &=
    -\eta_\rho G_t^\rho g_t^\rho
    -
    \rho_{t+1}G_t^\rho(1-\cos\theta_t).
\end{align}
Taking conditional expectation, the first term becomes
\begin{equation}
    \E_t[-\eta_\rho G_t^\rho g_t^\rho]
    =
    -\eta_\rho G_t^\rho\E_t[g_t^\rho]
    =
    -\eta_\rho(G_t^\rho)^2 .
\end{equation}
For the curvature contamination term,
\begin{align}
    -\rho_{t+1}G_t^\rho(1-\cos\theta_t)
    &\le
    |\rho_{t+1}G_t^\rho|(1-\cos\theta_t)\\
    &\le
    \rho_{\max}G_{\max}\frac{\theta_t^2}{2}\\
    &\le
    \frac{G_{\max}\rho_{\max}q_{\max}^2}{2}\eta_{\mathrm{dir}}^2 .
\end{align}
The last line uses $\theta_t=\eta_{\mathrm{dir}}\gamma_t q_t$, $\gamma_t\le1$, and the projected-kernel bound $q_t\le q_{\max}$.
Therefore
\begin{equation}
\label{app:eq:radial_descent}
    \E_t[I_\rho]
    \le
    -\eta_\rho(G_t^\rho)^2
    +
    \frac{G_{\max}\rho_{\max}q_{\max}^2}{2}\eta_{\mathrm{dir}}^2 .
\end{equation}

\subsection{First-order descent: tangent part}

The tangent component in \eqref{app:eq:first_order_split} is
\begin{equation}
    I_\perp =
    -\rho_{t+1}\sin\theta_t\langle G_t^\perp,D_t^\perp\rangle_\fro .
\end{equation}
Insert the identity $G_t^\perp=M_t^\perp-(M_t^\perp-G_t^\perp)$ and use $D_t^\perp=\widehat Q_t/q_t$:
\begin{align}
    I_\perp
    &=
    -\rho_{t+1}\frac{\sin\theta_t}{q_t}
    \langle M_t^\perp,\widehat Q_t\rangle_\fro
    +
    \rho_{t+1}\frac{\sin\theta_t}{q_t}
    \langle M_t^\perp-G_t^\perp,\widehat Q_t\rangle_\fro\\
    &\le
    -\rho_{t+1}\frac{\sin\theta_t}{q_t}c_{\mathrm{NS}}\|M_t^\perp\|_*
    +
    \rho_{t+1}\frac{\sin\theta_t}{q_t}
    \langle M_t^\perp-G_t^\perp,\widehat Q_t\rangle_\fro ,
\end{align}
where the second line uses Assumption~\ref{app:assump:polar}; the negative coefficient turns the alignment lower bound into an upper bound on $I_\perp$. By Assumption~\ref{app:assump:angle}, $\sin\theta_t\ge k\theta_t$ for the negative term. The remaining inner product can have either sign, so we first upper-bound it by its absolute value and then use $\sin\theta_t\le\theta_t$. Since $\theta_t=\eta_{\mathrm{dir}}\gamma_t q_t$, the actual post-projection norm cancels:
\begin{equation}
    I_\perp
    \le
    -k\rho_{t+1}\eta_{\mathrm{dir}}c_{\mathrm{NS}}\gamma_t\|M_t^\perp\|_*
    +
    \rho_{t+1}\eta_{\mathrm{dir}}\gamma_t
    \left|\langle M_t^\perp-G_t^\perp,\widehat Q_t\rangle_\fro\right| .
\end{equation}
We next lower bound the possibly damped and finite-NS-aligned nuclear norm. Since $\gamma_t\ge\gamma_{\min}$, $\|M\|_*\ge\|M\|_\fro$, and $\|M_t^\perp\|_\fro\le M_{\max}$,
\begin{equation}
    c_{\mathrm{NS}}\gamma_t\|M_t^\perp\|_*
    \ge
    c_{\mathrm{NS}}\gamma_{\min}\|M_t^\perp\|_\fro
    \ge
    \frac{c_{\mathrm{NS}}\gamma_{\min}}{M_{\max}}\|M_t^\perp\|_\fro^2 .
\end{equation}
Using $\rho_{t+1}\ge \rho_{\min}$ and defining
\begin{equation}
    \alpha_{\mathrm{dir}}=\frac{kc_{\mathrm{NS}}\rho_{\min}\gamma_{\min}}{M_{\max}},
\end{equation}
the main tangent term is bounded by
\begin{equation}
    -k\rho_{t+1}\eta_{\mathrm{dir}}c_{\mathrm{NS}}\gamma_t\|M_t^\perp\|_*
    \le
    -\alpha_{\mathrm{dir}}\eta_{\mathrm{dir}}\|M_t^\perp\|_\fro^2 .
\end{equation}
For the tracking-error term, Cauchy-Schwarz gives
\begin{align}
    \rho_{t+1}\eta_{\mathrm{dir}}\gamma_t
    \left|\langle M_t^\perp-G_t^\perp,\widehat Q_t\rangle_\fro\right|
    &\le
    \rho_{\max}\eta_{\mathrm{dir}}\|M_t^\perp-G_t^\perp\|_\fro\|\widehat Q_t\|_\fro\\
    &\le
    \rho_{\max}q_{\max}\eta_{\mathrm{dir}}\|M_t^\perp-G_t^\perp\|_\fro .
\end{align}
Applying Young's inequality $ab\le a^2/2+b^2/2$ with $a=\rho_{\max}q_{\max}\eta_{\mathrm{dir}}$ and $b=\|M_t^\perp-G_t^\perp\|_\fro$,
\begin{equation}
    \le
    \frac{1}{2}\rho_{\max}^2q_{\max}^2\eta_{\mathrm{dir}}^2
    +
    \frac{1}{2}\|M_t^\perp-G_t^\perp\|_\fro^2 .
\end{equation}
Taking conditional expectation,
\begin{equation}
\label{app:eq:tangent_descent}
    \E_t[I_\perp]
    \le
    -\alpha_{\mathrm{dir}}\eta_{\mathrm{dir}}\E_t[\|M_t^\perp\|_\fro^2]
    +
    \frac{1}{2}\E_t[\|M_t^\perp-G_t^\perp\|_\fro^2]
    +
    \frac{1}{2}\rho_{\max}^2q_{\max}^2\eta_{\mathrm{dir}}^2 .
\end{equation}

\subsection{Potential-energy inequality}

By smoothness,
\begin{equation}
    \E_t[f(W_{t+1})]-f(W_t)
    \le
    \E_t[\langle G_t,\Delta W_t\rangle_\fro]
    +
    \frac{L}{2}\E_t[\|\Delta W_t\|_\fro^2].
\end{equation}
Substituting \eqref{app:eq:second_order_bound}, \eqref{app:eq:radial_descent}, and \eqref{app:eq:tangent_descent}, then taking total expectation and using the tower property, we obtain
\begin{align}
    \E[f(W_{t+1})]-\E[f(W_t)]
    \le
    &-\left(\eta_\rho-\frac{L}{2}\eta_\rho^2\right)\E[(G_t^\rho)^2]
    -
    \alpha_{\mathrm{dir}}\eta_{\mathrm{dir}}\E[\|M_t^\perp\|_\fro^2]\nonumber\\
    &+
    \frac{1}{2}\E[\|M_t^\perp-G_t^\perp\|_\fro^2]
    +
    C_1\eta_{\mathrm{dir}}^2
    +
    \frac{L}{2}\sigma_\rho^2\eta_\rho^2,
\label{app:eq:potential}
\end{align}
where
\begin{equation}
    C_1=
    \frac{1}{2}\rho_{\max}^2q_{\max}^2
    +
    \frac{L}{2}\rho_{\max}^2q_{\max}^2
    +
    \frac{1}{2}G_{\max}\rho_{\max}q_{\max}^2 .
\end{equation}
If $\eta_\rho<2/L$, then $\eta_\rho-\frac{L}{2}\eta_\rho^2>0$.

\subsection{Tangent tracking-error dynamics}

Define the tangent tracking error
\begin{equation}
    E_t^\perp=\E[\|M_t^\perp-G_t^\perp\|_\fro^2].
\end{equation}
Here $M_t^\perp$ is the normalized analysis momentum defined in Appendix~\ref{app:global_convergence}; the implemented projected raw tangent buffer $M_{\mathrm{raw},t}^\perp$ differs from it only by the positive scalar $(1-\beta)^{-1}$ before Newton--Schulz normalization. For the actual raw-buffer recurrence, the normalized analysis momentum obeys
\begin{equation}
    M_{t+1}^\perp
    =
    \beta\mathcal P_{t+1}(M_t^\perp)
    +
    (1-\beta)g_{t+1}^\perp
    +
    \bar r_{t+1}.
\end{equation}
Subtract $G_{t+1}^\perp$ and add/subtract $\mathcal P_{t+1}(G_t^\perp)$:
\begin{align}
    M_{t+1}^\perp-G_{t+1}^\perp
    &=
    \beta\mathcal P_{t+1}(M_t^\perp-G_t^\perp)
    +
    \beta(\mathcal P_{t+1}(G_t^\perp)-G_{t+1}^\perp)\nonumber\\
    &\quad
    +
    (1-\beta)(g_{t+1}^\perp-G_{t+1}^\perp)
    +
    \bar r_{t+1}.
\end{align}
Let
\begin{align}
    x_t&=\beta\mathcal P_{t+1}(M_t^\perp-G_t^\perp),&
    y_t&=\beta(\mathcal P_{t+1}(G_t^\perp)-G_{t+1}^\perp),\\
    z_t&=(1-\beta)(g_{t+1}^\perp-G_{t+1}^\perp),&
    w_t&=\bar r_{t+1}.
\end{align}
Using
\begin{equation}
    \|x+y+z+w\|^2
    \le
    (1+\delta)\|x\|^2
    +
    3\left(1+\frac{1}{\delta}\right)(\|y\|^2+\|z\|^2+\|w\|^2),
\end{equation}
we obtain
\begin{align}
    E_{t+1}^\perp
    &\le
    \beta^2(1+\delta)
    \E[\|\mathcal P_{t+1}(M_t^\perp-G_t^\perp)\|_\fro^2]\nonumber\\
    &\quad+
    3\beta^2\left(1+\frac{1}{\delta}\right)
    \E[\|\mathcal P_{t+1}(G_t^\perp)-G_{t+1}^\perp\|_\fro^2]\nonumber\\
    &\quad+
    3(1-\beta)^2\left(1+\frac{1}{\delta}\right)\sigma_\perp^2
    \nonumber\\
    &\quad+
    3\left(1+\frac{1}{\delta}\right)\E[\|\bar r_{t+1}\|_\fro^2] .
\end{align}
Orthogonal projection is non-expansive:
\begin{equation}
    \|\mathcal P_{t+1}(M_t^\perp-G_t^\perp)\|_\fro^2
    \le
    \|M_t^\perp-G_t^\perp\|_\fro^2 .
\end{equation}
Choose $\delta=(1-\beta)/(2\beta)$. Then
\begin{equation}
    \beta^2(1+\delta)
    =
    \beta^2\left(1+\frac{1-\beta}{2\beta}\right)
    =
    \beta\frac{1+\beta}{2}
    <\beta,
\end{equation}
and
\begin{equation}
    3\left(1+\frac{1}{\delta}\right)
    =
    3\left(1+\frac{2\beta}{1-\beta}\right)
    =
    3\frac{1+\beta}{1-\beta}
    \le
    \frac{6}{1-\beta}.
\end{equation}
By Assumption~\ref{app:assump:tangent_lip}, the displacement bound, and \eqref{app:eq:leakage_bound},
\begin{align}
    E_{t+1}^\perp
    \le
    \beta E_t^\perp
    +
    \frac{6\beta^2L^2}{1-\beta}
    \E[\eta_\rho^2(g_t^\rho)^2+\rho_{\max}^2q_{\max}^2\eta_{\mathrm{dir}}^2]
    +
    6(1-\beta)\sigma_\perp^2
    +
    \frac{6K_{\mathrm{leak}}}{1-\beta}\eta_{\mathrm{dir}}^2 .
\end{align}
Using $\E[(g_t^\rho)^2]\le \E[(G_t^\rho)^2]+\sigma_\rho^2$,
\begin{align}
    E_{t+1}^\perp-E_t^\perp
    \le
    &-(1-\beta)E_t^\perp
    +
    \frac{6\beta^2L^2}{1-\beta}\eta_\rho^2\E[(G_t^\rho)^2]
    +
    C_2\eta_{\mathrm{dir}}^2\nonumber\\
    &+
    6(1-\beta)\sigma_\perp^2
    +
    \frac{6\beta^2L^2}{1-\beta}\sigma_\rho^2\eta_\rho^2 ,
\label{app:eq:kinetic}
\end{align}
where
\begin{equation}
    C_2=\frac{6\beta^2L^2\rho_{\max}^2q_{\max}^2}{1-\beta}
    +
    \frac{6K_{\mathrm{leak}}}{1-\beta}.
\end{equation}

\subsection{Lyapunov synthesis and convergence neighborhood}

Define the Lyapunov function
\begin{equation}
    \Phi_t=\E[f(W_t)]+c_\perp E_t^\perp .
\end{equation}
Choose
\begin{equation}
    c_\perp=\frac{1}{1-\beta},
\end{equation}
so that the $E_t^\perp$ coefficient after combining \eqref{app:eq:potential} and $c_\perp$ times \eqref{app:eq:kinetic} becomes
\begin{equation}
    \frac{1}{2}-c_\perp(1-\beta)=-\frac{1}{2}.
\end{equation}
The radial coefficient becomes
\begin{align}
    -\left(\eta_\rho-\frac{L}{2}\eta_\rho^2\right)
    +
    c_\perp\frac{6\beta^2L^2}{1-\beta}\eta_\rho^2
    &=
    -\eta_\rho
    +
    \frac{L}{2}\eta_\rho^2
    +
    \frac{6\beta^2L^2}{(1-\beta)^2}\eta_\rho^2\\
    &=
    -\eta_\rho
    \left(
    1-\frac{L}{2}\eta_\rho
    -
    \frac{6\beta^2L^2}{(1-\beta)^2}\eta_\rho
    \right).
\end{align}
For sufficiently small $\eta_\rho$, the parenthesized term is positive. Denote the resulting positive constant by $\widehat \alpha_\rho>0$. Combining all constant terms, including the raw-buffer leakage contribution $c_\perp C_2\eta_{\mathrm{dir}}^2$, into $\mathcal K_{\mathrm{curv}}$, $\mathcal K_n^\perp$, and $\mathcal K_n^\rho$ gives
\begin{align}
    \Phi_{t+1}-\Phi_t
    \le
    &-\widehat\alpha_\rho\eta_\rho\E[(G_t^\rho)^2]
    -
    \alpha_{\mathrm{dir}}\eta_{\mathrm{dir}}\E[\|M_t^\perp\|_\fro^2]
    -
    \frac{1}{2}E_t^\perp\nonumber\\
    &+
    \mathcal K_{\mathrm{curv}}\eta_{\mathrm{dir}}^2
    +
    \mathcal K_n^\perp\sigma_\perp^2
    +
    \mathcal K_n^\rho\sigma_\rho^2\eta_\rho^2 .
\label{app:eq:lyapunov_step}
\end{align}
Summing from $t=1$ to $T$,
\begin{align}
    \Phi_{T+1}-\Phi_1
    \le
    \sum_{t=1}^T
    \left[
    -\widehat\alpha_\rho\eta_\rho\E[(G_t^\rho)^2]
    -
    \alpha_{\mathrm{dir}}\eta_{\mathrm{dir}}\E[\|M_t^\perp\|_\fro^2]
    -
    \frac{1}{2}E_t^\perp
    +
    \mathcal K_{\mathrm{curv}}\eta_{\mathrm{dir}}^2
    +
    \mathcal K_n^\perp\sigma_\perp^2
    +
    \mathcal K_n^\rho\sigma_\rho^2\eta_\rho^2
    \right].
\end{align}
Since $f$ is lower bounded and $E_t^\perp\ge0$, there is a finite $\Phi_{\inf}$ with $\Phi_t\ge\Phi_{\inf}$. Moving the nonnegative descent terms to the left and using $\Phi_{T+1}\ge\Phi_{\inf}$ gives the finite-time bound
\begin{align}
    \frac{1}{T}\sum_{t=1}^T
    \left[
    \widehat\alpha_\rho\eta_\rho\E[(G_t^\rho)^2]
    +
    \alpha_{\mathrm{dir}}\eta_{\mathrm{dir}}\E[\|M_t^\perp\|_\fro^2]
    +
    \frac{1}{2}E_t^\perp
    \right]
    \le
    \mathcal K_{\mathrm{curv}}\eta_{\mathrm{dir}}^2
    +
    \mathcal K_n^\perp\sigma_\perp^2
    +
    \mathcal K_n^\rho\sigma_\rho^2\eta_\rho^2
    +
    \frac{\Phi_1-\Phi_{\inf}}{T}.
\label{app:eq:average_bound}
\end{align}
Taking $\limsup_{T\to\infty}$ removes the last term.

\begin{theorem}[Interior-regime decoupled convergence neighborhood]
\label{app:thm:neighborhood}
Under Assumptions~\ref{app:assump:smooth}--\ref{app:assump:angle}, sufficiently small $\eta_\rho$, inactive radial floors, and the nondegenerate raw-buffer recurrence above, the time-averaged radial gradient, normalized tangent momentum, and tangent tracking error admit the following asymptotic upper bounds:
\begin{align}
    \limsup_{T\rightarrow\infty}\frac{1}{T}\sum_{t=1}^T\E[(G_t^\rho)^2]
    &\le
    \frac{\mathcal K_{\mathrm{curv}}}{\widehat\alpha_\rho}\frac{\eta_{\mathrm{dir}}^2}{\eta_\rho}
    +
    \frac{\mathcal K_n^\perp}{\widehat\alpha_\rho}\frac{\sigma_\perp^2}{\eta_\rho}
    +
    \frac{\mathcal K_n^\rho}{\widehat\alpha_\rho}\eta_\rho\sigma_\rho^2
    =:\epsilon_\rho,\\
    \limsup_{T\rightarrow\infty}\frac{1}{T}\sum_{t=1}^T\E[\|M_t^\perp\|_\fro^2]
    &\le
    \frac{\mathcal K_{\mathrm{curv}}}{\alpha_{\mathrm{dir}}}\eta_{\mathrm{dir}}
    +
    \frac{\mathcal K_n^\perp}{\alpha_{\mathrm{dir}}}\frac{\sigma_\perp^2}{\eta_{\mathrm{dir}}}
    +
    \frac{\mathcal K_n^\rho}{\alpha_{\mathrm{dir}}}\frac{\eta_\rho^2}{\eta_{\mathrm{dir}}}\sigma_\rho^2,\\
    \limsup_{T\rightarrow\infty}\frac{1}{T}\sum_{t=1}^T E_t^\perp
    &\le
    2\mathcal K_{\mathrm{curv}}\eta_{\mathrm{dir}}^2
    +
    2\mathcal K_n^\perp\sigma_\perp^2
    +
    2\mathcal K_n^\rho\sigma_\rho^2\eta_\rho^2 .
\end{align}
Consequently,
\begin{equation}
    \limsup_{T\rightarrow\infty}\frac{1}{T}\sum_{t=1}^T\E[\|G_t^\perp\|_\fro^2]
    \le
    \mathcal O(\eta_{\mathrm{dir}})
    +
    \mathcal O\left(\frac{\sigma_\perp^2}{\eta_{\mathrm{dir}}}\right)
    +
    \mathcal O\left(\frac{\eta_\rho^2}{\eta_{\mathrm{dir}}}\sigma_\rho^2\right)
    =:\epsilon_\perp .
\end{equation}
\end{theorem}

\begin{proof}
The first three bounds are obtained by dividing \eqref{app:eq:average_bound} by $\widehat\alpha_\rho\eta_\rho$, $\alpha_{\mathrm{dir}}\eta_{\mathrm{dir}}$, and $1/2$, respectively. For the final true tangent-gradient bound, use
\begin{equation}
    \|G_t^\perp\|_\fro^2
    =
    \|M_t^\perp-(M_t^\perp-G_t^\perp)\|_\fro^2
    \le
    2\|M_t^\perp\|_\fro^2
    +
    2\|M_t^\perp-G_t^\perp\|_\fro^2,
\end{equation}
then take expectation and substitute the tangent-momentum and tracking-error bounds above.
\end{proof}

\section{PL Recursion for an Aligned First-Order Model}
\label{app:pl_convergence}

This appendix analyzes the first-order model $W_{t+1}=W_t-\eta H_t$ for RODE's effective displacement. It is separate from the exact raw-buffer recurrence proved in Appendix~\ref{app:global_convergence}: finite geodesic curvature, unequal coordinate stepsizes, the radial floor, and finite projected Newton--Schulz error must be absorbed into $H_t$ or controlled as residual terms before this model applies. Within this scope, boundedness, effective alignment, smoothness, and the PL condition yield linear convergence to a stepsize-controlled neighborhood.

\subsection{Additional assumptions for PL analysis}

\begin{assumption}[PL smooth setting]
\label{app:assump:pl}
The objective is $L$-smooth and satisfies the Polyak-Lojasiewicz condition
\begin{equation}
    \frac{1}{2}\|\nabla f(W_t)\|_\fro^2
    \ge
    \mu(f(W_t)-f^*).
\end{equation}
\end{assumption}

\begin{assumption}[Bounded effective direction]
\label{app:assump:direction_bound}
The full effective update direction $H_t$ satisfies $\|H_t\|_\fro\le G$.
\end{assumption}

\begin{assumption}[Local reprojected tangent-gradient regularity]
\label{app:assump:pl_tangent_lip}
For the momentum history used in the PL argument, let $G_i^{\mathrm{tan}}$ denote the tangent gradient from step $i$ reprojected into the current tangent space at $U_t$, and let $G_{\mathrm{tan}}=G_t^\perp$. There is a constant $L_{\mathrm{tan}}$ such that
\begin{equation}
    \|G_i^{\mathrm{tan}}-G_{\mathrm{tan}}\|_\fro
    \le
    L_{\mathrm{tan}}\|W_i-W_t\|_\fro .
\end{equation}
This local condition accounts for both ordinary gradient variation and projection drift from moving tangent spaces; it is not implied by Euclidean $L$-smoothness alone.
\end{assumption}

\begin{assumption}[Bounded tangent gradient and local polar stability]
\label{app:assump:polar_stability}
Along the trajectory considered in the PL argument, $\|G_{\mathrm{tan}}\|_\fro\le G_{\max}$. Moreover, the projected polar map is locally Lipschitz around the tangent gradients: for the momentum tracking errors considered below,
\begin{equation}
    \|\mathcal P_t(\Polar(G_{\mathrm{tan}}+E))-\mathcal P_t(\Polar(G_{\mathrm{tan}}))\|_\fro
    \le
    C_{\mathrm{pol}}\|E\|_\fro .
\end{equation}
This assumption excludes degenerate points where the polar factor is unstable. Since $\mathcal P_t$ is non-expansive, the displayed condition follows from the corresponding local Lipschitz bound for the unprojected polar map; the global-neighborhood proof above instead isolates polar accuracy through Assumption~\ref{app:assump:polar}.
\end{assumption}

\begin{assumption}[Strong directional alignment]
\label{app:assump:alignment}
There exists $\alpha>0$ such that
\begin{equation}
    \langle G_t,H_t\rangle_\fro
    \ge
    \alpha\|G_t\|_\fro^2.
\end{equation}
\end{assumption}

The next two lemmas give sufficient alignment conditions for a projected-polar first-order model. They do not establish the condition for every finite-step implementation branch; Theorem~\ref{app:thm:pl} assumes effective alignment directly.

\subsection{Radial alignment}

\begin{lemma}[Radial alignment]
\label{app:lem:radial_alignment}
The radial update direction has positive alignment with the radial gradient.
\end{lemma}

\begin{proof}
Write $W_t=\rho_tU_t$. The radial gradient vector is
\begin{equation}
    G_{\mathrm{rad}}=\langle G_t,U_t\rangle_\fro U_t .
\end{equation}
The ideal radial descent direction is this gradient component; the outer PL recursion applies the scalar stepsize separately:
\begin{equation}
    H_{\mathrm{rad}}=G_{\mathrm{rad}} .
\end{equation}
Therefore
\begin{align}
    \langle G_{\mathrm{rad}},H_{\mathrm{rad}}\rangle_\fro
    &=
    \|G_{\mathrm{rad}}\|_\fro^2 .
\end{align}
Let $K_{\mathrm{rad}}=1$. Then
\begin{equation}
    \langle G_{\mathrm{rad}},H_{\mathrm{rad}}\rangle_\fro
    \ge
    K_{\mathrm{rad}}\|G_{\mathrm{rad}}\|_\fro^2 .
\end{equation}
\end{proof}

\subsection{Tangent alignment}

\begin{lemma}[Tangent alignment in the projected-polar model]
\label{app:lem:tangent_alignment}
Under Assumptions~\ref{app:assump:direction_bound}--\ref{app:assump:polar_stability} and sufficiently small learning rate, the projected-polar tangent step in the first-order model remains positively aligned with the true tangent gradient. The same result holds for scalar damping bounded below by a positive constant.
\end{lemma}

\begin{proof}
Let $G_{\mathrm{tan}}$ denote the tangent component of $G_t$ at the current direction $U_t$. The tangent update is based on momentum. To show that momentum preserves alignment, use the projection-based momentum form of the EMA: $G_i^{\mathrm{tan}}$ denotes a past tangent gradient reprojected into the current tangent space, and
\begin{equation}
    M_t=\sum_{i=0}^{t} w_iG_i^{\mathrm{tan}},
\end{equation}
where, with initialization at step $0$,
\begin{equation}
    w_i=
    \begin{cases}
        (1-\beta)\beta^{t-i}, & i\ge 1,\\
        \beta^t, & i=0 .
    \end{cases}
\end{equation}
The weights sum to one:
\begin{align}
    \sum_{i=0}^{t}w_i
    &=
    \left(\sum_{i=1}^{t}(1-\beta)\beta^{t-i}\right)+\beta^t\\
    &=
    (1-\beta^t)+\beta^t
    =
    1.
\end{align}
Define the momentum tracking error $E_t=M_t-G_{\mathrm{tan}}$. Because $\sum_iw_i=1$,
\begin{align}
    E_t
    &=
    \sum_{i=0}^{t}w_iG_i^{\mathrm{tan}}-\sum_{i=0}^{t}w_iG_{\mathrm{tan}}\\
    &=
    \sum_{i=0}^{t}w_i(G_i^{\mathrm{tan}}-G_{\mathrm{tan}}).
\end{align}
By the triangle inequality,
\begin{equation}
    \|E_t\|_\fro
    \le
    \sum_{i=0}^{t}w_i\|G_i^{\mathrm{tan}}-G_{\mathrm{tan}}\|_\fro .
\end{equation}
Using Assumption~\ref{app:assump:pl_tangent_lip},
\begin{equation}
    \|G_i^{\mathrm{tan}}-G_{\mathrm{tan}}\|_\fro
    \le
    L_{\mathrm{tan}}\|W_i-W_t\|_\fro .
\end{equation}
If each update has norm at most $G$ and learning rate $\eta$, then
\begin{align}
    \|W_i-W_t\|_\fro
    &=
    \left\|\sum_{k=i+1}^{t}\eta H_k\right\|_\fro\\
    &\le
    \sum_{k=i+1}^{t}\eta\|H_k\|_\fro\\
    &\le
    (t-i)\eta G .
\end{align}
Thus
\begin{equation}
    \|E_t\|_\fro
    \le
    L_{\mathrm{tan}}G\eta\sum_{i=0}^{t}w_i(t-i).
\end{equation}
Let
\begin{equation}
    S=\sum_{i=0}^{t}w_i(t-i)
    =
    t\beta^t+(1-\beta)\sum_{k=1}^{t-1}k\beta^k .
\end{equation}
For $A=\sum_{k=1}^{t-1}k\beta^k$,
\begin{equation}
    A=\beta+2\beta^2+\cdots+(t-1)\beta^{t-1},
\end{equation}
and
\begin{equation}
    \beta A=\beta^2+2\beta^3+\cdots+(t-2)\beta^{t-1}+(t-1)\beta^t .
\end{equation}
Subtracting,
\begin{align}
    (1-\beta)A
    &=
    \beta+\beta^2+\cdots+\beta^{t-1}-(t-1)\beta^t\\
    &=
    \frac{\beta(1-\beta^{t-1})}{1-\beta}-(t-1)\beta^t .
\end{align}
Substituting into $S$,
\begin{align}
    S
    &=
    t\beta^t+\frac{\beta-\beta^t}{1-\beta}-t\beta^t+\beta^t\\
    &=
    \frac{\beta-\beta^{t+1}}{1-\beta}
    <
    \frac{\beta}{1-\beta}.
\end{align}
Therefore
\begin{equation}
    \|E_t\|_\fro
    \le
    L_{\mathrm{tan}}G\eta\frac{\beta}{1-\beta}
    =
    C_1\eta .
\end{equation}
Now consider the projected polar factor. Let
\begin{equation}
    O_t=\mathcal P_t(\Polar(M_t))=\mathcal P_t(\Polar(G_{\mathrm{tan}}+E_t)),
    \qquad
    O_g=\mathcal P_t(\Polar(G_{\mathrm{tan}})).
\end{equation}
By Assumption~\ref{app:assump:polar_stability}, the polar factor is locally stable, so
\begin{equation}
    O_t=O_g+\Delta_o,
    \qquad
    \|\Delta_o\|_\fro\le C_{\mathrm{pol}}\|E_t\|_\fro\le C_1C_{\mathrm{pol}}\eta .
\end{equation}
The first-order tangent contribution to the effective direction is $H_{\mathrm{tan}}=\gamma_tO_t$, which is tangent by construction because of the projection $\mathcal P_t$. For default RODE, $\gamma_t=1$; for an optional damped variant, assume
\begin{equation}
    0<\gamma_{\min}\le \gamma_t\le 1 .
\end{equation}
Then
\begin{align}
    \langle G_{\mathrm{tan}},H_{\mathrm{tan}}\rangle_\fro
    &=
    \gamma_t
    \left(
    \langle G_{\mathrm{tan}},O_g\rangle_\fro
    +
    \langle G_{\mathrm{tan}},\Delta_o\rangle_\fro
    \right).
\end{align}
For the projected polar target, since $G_{\mathrm{tan}}$ lies in the current tangent space,
\begin{equation}
    \langle G_{\mathrm{tan}},O_g\rangle_\fro
    =
    \langle G_{\mathrm{tan}},\Polar(G_{\mathrm{tan}})\rangle_\fro
    =
    \|G_{\mathrm{tan}}\|_*
    \ge
    \|G_{\mathrm{tan}}\|_\fro .
\end{equation}
For the perturbation term,
\begin{equation}
    |\langle G_{\mathrm{tan}},\Delta_o\rangle_\fro|
    \le
    \|G_{\mathrm{tan}}\|_\fro\|\Delta_o\|_\fro
    \le
    C_1C_{\mathrm{pol}}\eta\|G_{\mathrm{tan}}\|_\fro .
\end{equation}
Hence
\begin{align}
    \langle G_{\mathrm{tan}},H_{\mathrm{tan}}\rangle_\fro
    &\ge
    \gamma_t(1-C_1C_{\mathrm{pol}}\eta)\|G_{\mathrm{tan}}\|_\fro\\
    &\ge
    \gamma_{\min}(1-C_1C_{\mathrm{pol}}\eta)\|G_{\mathrm{tan}}\|_\fro .
\end{align}
If $\|G_{\mathrm{tan}}\|_\fro\le G_{\max}$, then
\begin{equation}
    \|G_{\mathrm{tan}}\|_\fro
    \ge
    \frac{\|G_{\mathrm{tan}}\|_\fro^2}{G_{\max}} .
\end{equation}
Therefore
\begin{equation}
    \langle G_{\mathrm{tan}},H_{\mathrm{tan}}\rangle_\fro
    \ge
    \frac{\gamma_{\min}(1-C_1C_{\mathrm{pol}}\eta)}{G_{\max}}
    \|G_{\mathrm{tan}}\|_\fro^2 .
\end{equation}
For $\eta<1/(C_1C_{\mathrm{pol}})$, define
\begin{equation}
    K_{\mathrm{tan}}
    =
    \frac{\gamma_{\min}(1-C_1C_{\mathrm{pol}}\eta)}{G_{\max}}>0.
\end{equation}
This proves tangent alignment.
\end{proof}

\subsection{Orthogonal synthesis}

\begin{lemma}[Alignment synthesis in the first-order model]
\label{app:lem:global_alignment}
The radial and tangent alignment bounds imply Assumption~\ref{app:assump:alignment}.
\end{lemma}

\begin{proof}
Decompose
\begin{equation}
    G_t=G_{\mathrm{rad}}+G_{\mathrm{tan}},
    \qquad
    H_t=H_{\mathrm{rad}}+H_{\mathrm{tan}}.
\end{equation}
The radial and tangent subspaces are orthogonal, so
\begin{equation}
    \langle G_{\mathrm{rad}},H_{\mathrm{tan}}\rangle_\fro=0,
    \qquad
    \langle G_{\mathrm{tan}},H_{\mathrm{rad}}\rangle_\fro=0,
\end{equation}
and
\begin{equation}
    \|G_t\|_\fro^2
    =
    \|G_{\mathrm{rad}}\|_\fro^2
    +
    \|G_{\mathrm{tan}}\|_\fro^2 .
\end{equation}
Thus
\begin{align}
    \langle G_t,H_t\rangle_\fro
    &=
    \langle G_{\mathrm{rad}},H_{\mathrm{rad}}\rangle_\fro
    +
    \langle G_{\mathrm{tan}},H_{\mathrm{tan}}\rangle_\fro\\
    &\ge
    K_{\mathrm{rad}}\|G_{\mathrm{rad}}\|_\fro^2
    +
    K_{\mathrm{tan}}\|G_{\mathrm{tan}}\|_\fro^2 .
\end{align}
Let $K=\min\{K_{\mathrm{rad}},K_{\mathrm{tan}}\}>0$. Then
\begin{equation}
    \langle G_t,H_t\rangle_\fro
    \ge
    K(\|G_{\mathrm{rad}}\|_\fro^2+\|G_{\mathrm{tan}}\|_\fro^2)
    =
    K\|G_t\|_\fro^2 .
\end{equation}
Taking $\alpha=K$ proves the claim.
\end{proof}

\subsection{PL linear convergence}

\begin{theorem}[PL recursion for an aligned first-order model]
\label{app:thm:pl}
Assume smoothness, the PL condition, bounded updates, directional alignment for the first-order update $W_{t+1}=W_t-\eta H_t$, and $0<2\eta\alpha\mu<1$. Then this aligned effective update satisfies
\begin{equation}
    f(W_t)-f^*
    \le
    (1-2\eta\alpha\mu)^t(f(W_0)-f^*)
    +
    \frac{K_{\mathrm{err}}\eta}{2\alpha\mu},
\end{equation}
where $K_{\mathrm{err}}=\frac{L}{2}G^2$.
\end{theorem}

\begin{proof}
Consider the first-order effective update
\begin{equation}
    W_{t+1}=W_t-\eta H_t.
\end{equation}
By $L$-smoothness,
\begin{align}
    f(W_{t+1})
    &\le
    f(W_t)
    +
    \langle \nabla f(W_t),W_{t+1}-W_t\rangle_\fro
    +
    \frac{L}{2}\|W_{t+1}-W_t\|_\fro^2\\
    &=
    f(W_t)
    -
    \eta\langle G_t,H_t\rangle_\fro
    +
    \frac{L}{2}\eta^2\|H_t\|_\fro^2 .
\end{align}
Using $\|H_t\|_\fro\le G$,
\begin{equation}
    f(W_{t+1})
    \le
    f(W_t)
    -
    \eta\langle G_t,H_t\rangle_\fro
    +
    K_{\mathrm{err}}\eta^2,
    \qquad
    K_{\mathrm{err}}=\frac{L}{2}G^2 .
\end{equation}
Directional alignment gives
\begin{equation}
    -\eta\langle G_t,H_t\rangle_\fro
    \le
    -\eta\alpha\|G_t\|_\fro^2.
\end{equation}
Therefore
\begin{equation}
    f(W_{t+1})
    \le
    f(W_t)-\eta\alpha\|G_t\|_\fro^2+K_{\mathrm{err}}\eta^2.
\end{equation}
Subtract $f^*$ from both sides:
\begin{equation}
    f(W_{t+1})-f^*
    \le
    f(W_t)-f^*
    -
    \eta\alpha\|G_t\|_\fro^2
    +
    K_{\mathrm{err}}\eta^2 .
\end{equation}
The PL condition implies
\begin{equation}
    \|G_t\|_\fro^2
    \ge
    2\mu(f(W_t)-f^*).
\end{equation}
Thus
\begin{align}
    f(W_{t+1})-f^*
    &\le
    f(W_t)-f^*
    -
    2\eta\alpha\mu(f(W_t)-f^*)
    +
    K_{\mathrm{err}}\eta^2\\
    &=
    (1-2\eta\alpha\mu)(f(W_t)-f^*)
    +
    K_{\mathrm{err}}\eta^2 .
\end{align}
Let $\Delta_t=f(W_t)-f^*$ and $\rho=2\eta\alpha\mu$. The recursion is
\begin{equation}
    \Delta_{t+1}\le (1-\rho)\Delta_t+K_{\mathrm{err}}\eta^2.
\end{equation}
Expanding the first few steps makes the geometric structure explicit:
\begin{align}
    \Delta_1
    &\le
    (1-\rho)\Delta_0+K_{\mathrm{err}}\eta^2,\\
    \Delta_2
    &\le
    (1-\rho)^2\Delta_0
    +(1-\rho)K_{\mathrm{err}}\eta^2
    +K_{\mathrm{err}}\eta^2,\\
    \Delta_3
    &\le
    (1-\rho)^3\Delta_0
    +(1-\rho)^2K_{\mathrm{err}}\eta^2
    +(1-\rho)K_{\mathrm{err}}\eta^2
    +K_{\mathrm{err}}\eta^2 .
\end{align}
By induction,
\begin{equation}
    \Delta_t
    \le
    (1-\rho)^t\Delta_0
    +
    K_{\mathrm{err}}\eta^2
    \sum_{i=0}^{t-1}(1-\rho)^i .
\end{equation}
Since
\begin{equation}
    \sum_{i=0}^{t-1}(1-\rho)^i
    =
    \frac{1-(1-\rho)^t}{\rho}
    <
    \frac{1}{\rho},
\end{equation}
we obtain
\begin{align}
    \Delta_t
    &\le
    (1-\rho)^t\Delta_0
    +
    \frac{K_{\mathrm{err}}\eta^2}{\rho}\\
    &=
    (1-2\eta\alpha\mu)^t(f(W_0)-f^*)
    +
    \frac{K_{\mathrm{err}}\eta}{2\alpha\mu}.
\end{align}
This proves linear convergence to an $\mathcal O(\eta)$ neighborhood for the aligned first-order model. It does not extend the PL rate automatically to every branch of the finite implementation; Appendix~\ref{app:global_convergence} provides the separate convergence-neighborhood guarantee for the exact raw-buffer update under its stated assumptions.
\end{proof}

\end{document}